\def\year{2020}\relax
\documentclass[letterpaper]{article} % DO NOT CHANGE THIS
\usepackage{aaai20}  % DO NOT CHANGE THIS
\usepackage{times}  % DO NOT CHANGE THIS
\usepackage{helvet} % DO NOT CHANGE THIS
\usepackage{courier}  % DO NOT CHANGE THIS
\usepackage[hyphens]{url}  % DO NOT CHANGE THIS
\usepackage{graphicx} % DO NOT CHANGE THIS
\usepackage{graphicx}  % DO NOT CHANGE THIS
\author{Hiroshi Kera, Yoshihiko Hasegawa\\
Department of Information and Communication Engineering, \\
Graduate School of Information Science and Technology, \\
The University of Tokyo, Tokyo, Japan\\
}

\title{Gradient Boosts the Approximate Vanishing Ideal}
\usepackage{amsmath}
\usepackage{amsthm}
\usepackage{amssymb}
\usepackage{comment}
\usepackage{appendix}
\usepackage{multirow}
\newtheorem{thm}{Theorem}
\newtheorem{defn}{Definition}
\newtheorem{prop}{Proposition}

\newtheorem{lem}{Lemma}
\newtheorem{rem}{Remark}

\newtheorem{conj}{Conjecture}

\usepackage[whole]{bxcjkjatype}
\usepackage[normalem]{ulem}
\usepackage{epsfig, color}

\begin{document}
\maketitle
\begin{abstract}
In the last decade, the approximate vanishing ideal and its basis construction algorithms have been extensively studied in computer algebra and machine learning as a general model to reconstruct the algebraic variety on which noisy data approximately lie.
In particular, the basis construction algorithms developed in machine learning are widely used in applications across many fields because of their monomial-order-free property; however, they lose many of the theoretical properties of computer-algebraic algorithms. 
In this paper, we propose general methods that equip monomial-order-free algorithms with several advantageous theoretical properties. Specifically, we exploit the gradient to (i) sidestep the spurious vanishing problem in polynomial time to remove symbolically trivial redundant bases, (ii) achieve consistent output with respect to the translation and scaling of input, and (iii) remove nontrivially redundant bases. The proposed methods work in a fully numerical manner, whereas existing algorithms require the awkward monomial order or exponentially costly (and mostly symbolic) computation to realize properties (i) and (iii). To our knowledge, property (ii) has not been achieved by any existing basis construction algorithm of the approximate vanishing ideal. 
\end{abstract}

\section{Introduction}
A set of data points lies in an algebraic variety\footnote{An algebraic variety here refers to a set of points that can be described as the solutions of a polynomial system and it is not necessarily irreducible.
}---this is a common assumption in various methods in machine learning. For example, linear analysis methods such as principal component analysis are designed to work with data lying in linear subspace, which is a class of algebraic varieties. Broader classes of algebraic varieties are considered in subspace clustering~\cite{vidal2005generalized}, matrix completion~\cite{ongie2017algebraic,li2017riemannian}, and classification~\cite{livni2013vanishing,globerson17effective}. 
In the last decade, the approximate vanishing ideal~\cite{heldt2009approximate} has been considered for the machine-learning problem in the most general setting, namely, retrieving a polynomial system that describes the algebraic variety where noisy data points approximately lie. 
An approximate vanishing ideal of a set of points $X\subset\mathbb{R}^n$ is a set of approximate vanishing polynomials, each of which almost takes a zero value for any $\boldsymbol{x}\in X$. Roughly, 
\begin{align*}
\mathcal{I}_{\mathrm{app}}(X) &= \left\{ g\in\mathcal{P}_n\mid \forall \boldsymbol{x}\in X, g(\boldsymbol{x})\approx 0 \right\},
\end{align*}
where $\mathcal{P}_n$ is the set of all $n$-variate polynomials over $\mathbb{R}$. Various basis construction algorithms for the approximate vanishing ideal have been proposed, first in computer algebra and then in machine learning~\cite{heldt2009approximate,fassino2010almost,livni2013vanishing,limbeck2014computation,kiraly2014dual,kera2018approximate}. However, existing algorithms suffer from the tradeoff between practicality and theoretical soundness. Basis construction algorithms developed in machine learning are more practically convenient and are used in various fields~\cite{zhao2014hand,hou2016discriminative,kera2016vanishing,iraji2017principal,wang2018nonlinear}. This is because these algorithms work with numerical computation and without the monomial order, which is a prefixed prioritization of monomials. Different monomial orders can yield different results, but it is unknown how to properly select a monomial order from exponentially many candidates. However, while enjoying the monomial-order-free property, the basis construction algorithms in machine learning lack various theoretical (and advantageous) properties of computer-algebraic algorithms, which use the practically awkward monomial order and symbolic computation.

In this paper, we propose general and efficient methods that enable monomial-order-free basis construction algorithms in machine learning to have various advantageous theoretical properties. In particular, we address the three theoretical issues listed below. To our knowledge, none of the existing basis construction algorithms can resolve the first and third issues in polynomial time without using a monomial order. Furthermore, the second issue has not been addressed by any existing basis construction methods.

\noindent\textbf{The spurious vanishing problem}---A polynomial $g$ can approximately vanish for a point $\boldsymbol{x}$, i.e., $g(\boldsymbol{x})\approx 0$ not because $\boldsymbol{x}$ is close to the roots of $g$ but merely because $g$ is close to the zero polynomial (i.e., the coefficients of the monomials in $g$ are all small)\footnote{For example, a univariate polynomial $g = x^2-1$ approximately vanishes only for points close to its roots $x=\pm 1$. However, once $g$ is scaled to $kg$ by a small nonzero $k\in\mathbb{R}$, then $kg$ can approximately vanish for points far from its roots. A simple remedy is to normalize $kg$ as $kg/\sqrt{2}k$ using the coefficients.}. To sidestep this, $g$ needs to be normalized by some scale. However, intuitive coefficient normalization is exponentially costly for monomial-order-free algorithms~\cite{kera2019spurious}.

\noindent\textbf{Inconsistency of output with respect to the translation and scaling of input}---Given translated or scaled data, the output of the basis construction can drastically change in terms of the number of polynomials and their nonlinearity, regardless of how well the parameter is chosen.
This contradicts the intuition that the intrinsic structure of an algebraic variety does not change by a translation or scaling on data. 

\noindent\textbf{Redundancy in the basis set}---The output basis set can contain polynomials that are redundant because they can be generated by other lower-degree polynomials\footnote{For example, a polynomial $gh$ is unnecessary if $g$ is included in the basis set.}. Determining the redundancy usually needs exponentially costly symbolic procedures and is also unreliable in our approximate setting.

To efficiently address these issues without symbolic computation, we exploit the gradient of the polynomials at the input points, which has been rarely considered in the relevant literature. The advantages of this approach are that (i) gradient can be efficiently and exactly computed in our setting without differentiation and (ii) it provides some information on the symbolic structure of and symbolic relations between polynomials in a numerical manner.

In summary, we propose fully numerical methods for monomial-order-free algorithms to retain two theoretical properties of computer-algebraic algorithms and gain one new advantageous property. Hence, we exploit the gradient to address the aforementioned three fundamental issues as follows.
\begin{itemize}
    \item We propose gradient normalization to resolve the spurious vanishing problem. A polynomial is normalized by the norm of its gradients at the input points. This approach is based on the intuition that polynomials close to the zero polynomial have a small norm for the gradients at all locations. A rigorous theoretical analysis shows its validity. 
    \item We prove that by introducing gradient normalization, a standard basis construction algorithm can equip a sort of invariance to transformations (scaling and translation) of input data points. The number of basis polynomials at each degree is the same before and after the transformation and the change of each basis polynomial is analytically presented. 
    \item We propose a basis reduction method that considers the linear dependency of gradients between polynomials and removes redundant ones from a basis set without symbolic operations.
\end{itemize}

\section{Related Work}
Based on classical basis construction algorithms for noise-free points~\cite{moller1982construction,kehrein2006computing}, most algorithms of the approximate vanishing ideal in computer algebra efficiently sidestep the issues with the spurious vanishing problem and basis set redundancy using the monomial order and symbolic computation. 
To our knowledge, there are two algorithms that work without the monomial order in computer algebra~\cite{sauer2007approximate,hashemi2019computing}, but both require exponential-time procedures.
Although the gradient has been rarely considered in the basis construction of the (approximate) vanishing ideal, \citeauthor{fassino2010almost}~(\citeyear{fassino2010almost}) used the gradient during basis construction to check whether a given polynomial exactly vanishes after slightly perturbing given points. \citeauthor{vidal2005generalized} et al.~(\citeyear{vidal2005generalized}) considered a union of subspaces for clustering, where the gradient at some points are used to estimate the dimension of each subspace where a cluster lies. Both of these works use the gradient for purposes that are totally different from ours. The closest work to ours is \cite{fassino2013simple}, which proposes an algorithm to compute an approximate vanishing polynomial of low degree based on the geometrical distance using the gradient. However, their algorithm does not compute a basis set but only provide a single approximate vanishing polynomial. Furthermore, the computation relies on the monomial order and coefficient normalization.

\section{Preliminaries\label{sec:Preliminaries}}
Throughout the paper, a polynomial is represented as $h$ without arguments and $h(\boldsymbol{x})\in\mathbb{R}$ denotes the evaluation
of $h$ at a point $\boldsymbol{x}$.

\subsection{From polynomials to evaluation vectors}
\begin{defn}[Vanishing Ideal]\label{def:vanishing-ideal}
Given a set of $n$-dimensional points $X$, the vanishing ideal
of $X$ is a set of $n$-variate polynomials that take a zero value,
(i.e., vanish) for any point in $X$. Formally, 
\begin{align*}
\mathcal{I}(X) & =\left\{ g\in\mathcal{P}_{n}\mid\forall\boldsymbol{x}\in X,g(\boldsymbol{x})=0\right\}.
\end{align*}
\end{defn}
\begin{defn}[Evaluation vector and evaluation matrix]
Given a set of points $X = \{\boldsymbol{x}_1,\boldsymbol{x}_2,...,\boldsymbol{x}_{|X|}\}$, the evaluation vector of a polynomial $h$ is defined
as follows: 
\begin{align*}
h(X) & =\begin{pmatrix}h(\boldsymbol{x}_{1}) & h(\boldsymbol{x}_{2}) & \cdots & h(\boldsymbol{x}_{|X|})\end{pmatrix}^{\top}\in\mathbb{R}^{|X|},
\end{align*}
where $|\cdot|$ denotes the cardinality of a set.
For a set of polynomials $H=\left\{ h_{1},h_{2},\ldots,h_{|H|}\right\} $,
its evaluation matrix is $H(X)=(h_{1}(X)\ h_{2}(X)\ \cdots\ h_{|H|}(X))\in\mathbb{R}^{|X|\times |H|}$. 
\end{defn}
\begin{defn}[$\epsilon$-vanishing polynomial]
A polynomial $g$ is an $\epsilon$-vanishing polynomial for a set of points $X$ if $\|g(X)\|\le\epsilon$, where $\|\cdot\|$ denotes the Euclidean norm; otherwise, $g$ is an $\epsilon$-nonvanishing polynomial. 
\end{defn}
As Definition~\ref{def:vanishing-ideal} indicates, we are only interested in the evaluation values of polynomials for the given set of points $X$. Hence, a polynomial $h$ can be represented by its evaluation vector $h(X)$. As a consequence, the product and weighted sum of polynomials become linear algebra operations. Let us consider a set of polynomials $H=\{h_1,h_2,...,h_{|H|}\}$. A product of $h_1,h_2\in H$ becomes $h_1(X)\odot h_2(X)$, where $\odot$ denotes the entry-wise product. 
A weighted sum $\sum_{i=1}^{|X|} w_{i}h_i$, where $w_i\in\mathbb{R}$, becomes $\sum_{i=1}^{|H|} w_{i}h_i(X)$. The weighted sum of polynomials is an important building block in the following discussion. For convenience of notation, we define a special product between a polynomial set and a vector as $H\boldsymbol{w}:=\sum_{i=1}^{|H|}w_ih_i$, where  $w_i$ is the $i$-th entry of $\boldsymbol{w}\in\mathbb{R}^{|H|}$. Similarly, we denote the product between a polynomial set $H$ and a matrix $W=(\boldsymbol{w}_1\ \boldsymbol{w}_2\ \cdots \boldsymbol{w}_{s})\in\mathbb{R}^{|H|\times s}$ as $HW:=\{H\boldsymbol{w}_1,H\boldsymbol{w}_2,...,H\boldsymbol{w}_{s}\}$. Note that $(H\boldsymbol{w})(X) = H(X)\boldsymbol{w}$ and $(HW)(X) = H(X)W$. 
We consider a set of polynomials that is \textit{spanned} by a set $F$ of nonvanishing polynomials or \textit{generated} by a set $G$ of vanishing polynomials. We denote the former as $\mathrm{span}(F) = \{\sum_{f\in F}a_f f\mid a_f\in\mathbb{R}\}$ and the latter as $\langle G\rangle = \{\sum_{g\in G}h_g g\mid h_g\in\mathcal{P}_n\}$.

\subsection{Simple Basis Construction Algorithm}
Our idea of using the gradient is general enough to be integrated with existing monomial-order-free algorithms. However, to avoid a unnecessarily abstract discussion, we focus on the Simple Basis Construction (SBC) algorithm~\cite{kera2019spurious}, which was proposed by~\cite{kera2019spurious} based on Vanishing Component Analysis~(VCA; \citeauthor{livni2013vanishing}~\citeyear{livni2013vanishing}). 
Most monomial-order-free algorithms can be discussed using SBC; thus, the following discussion is sufficiently general. 

The input to SBC is a set of points $X\subset \mathbb{R}^{n}$ and error tolerance $\epsilon\ge 0$. SBC outputs a basis set $G$ of $\epsilon$-vanishing polynomials and a basis set of $\epsilon$-nonvanishing polynomials $F$. We later discuss the conditions that $G$ and $F$ are required to satisfy (cf., Theorem~\ref{thm:basis}). SBC proceeds from degree-0 polynomials to those of higher degree. At each degree $t$, a set of degree-$t$ $\epsilon$-vanishing polynomials $G_t$ and a set of degree-$t$ $\epsilon$-nonvanishing polynomials $F_t$ are generated. We use notations $F^{t}=\bigcup_{\tau=0}^t F_{\tau}$ and $G^{t}=\bigcup_{\tau=0}^t G_{\tau}$.
For $t=0$, $F_0 = \{m\}$ and $G_0 = \emptyset$, where $m\ne 0$ is a constant polynomial. 
At each degree $t\ge 1$, the following procedures (\texttt{Step 1}, \texttt{Step 2}, and \texttt{Step 3}) are conducted\footnote{For ease of understanding, we describe the procedures in the form of symbolic computation, but these can be numerically implemented (i.e., by matrix-vector calculations)}.

\paragraph*{Step 1: Generate a set of candidate polynomials}
Pre-candidate polynomials of degree $t$ for $t> 1$ are generated by multiplying nonvanishing polynomials across $F_1$ and $F_{t-1}$.
\begin{align*}
    C_t^{\mathrm{pre}} = \{pq \mid p\in F_1, q\in F_{t-1}\}.
\end{align*}
At $t=1$, $C_1^{\mathrm{pre}}=\{x_1,x_2,...,x_n\}$, where $x_k$ are variables. The candidate basis is then generated through the orthogonalization.
\begin{align}\label{eq:orthogonalization}
     C_{t} &= C_{t}^{\mathrm{pre}} - F^{t-1}F^{t-1}(X)^{\dagger}C_{t}^{\mathrm{pre}}(X),
 \end{align}
 where $\cdot^{\dagger}$ is the pseudo-inverse of a matrix.
\paragraph*{Step 2: Solve a generalized eigenvalue problem}{
We solve the following generalized eigenvalue problem: 
\begin{align}\label{eq:gep}
    C_t(X)^{\top}C_t(X)V = \mathfrak{N}(C_t)V\Lambda,
\end{align}
where a matrix $V$ that has generalized  eigenvectors $\boldsymbol{v}_1,\boldsymbol{v}_2,...,\boldsymbol{v}_{|C_t|}$ for its columns, $\Lambda$ is a diagonal matrix with generalized eigenvalues $\lambda_1,\lambda_2,...,\lambda_{|C_t|}$ along its diagonal, and $\mathfrak{N}(C_t)\in\mathbb{R}^{|C_t|\times |C_t|}$ is the normalization matrix, which will soon be introduced.
}
\paragraph*{Step 3: Construct sets of basis polynomials}{
Basis polynomials are generated by linearly combining polynomials in $C_t$ with $\{\boldsymbol{v}_1,\boldsymbol{v}_2,...,\boldsymbol{v}_{|C_t|}\}$.
\begin{align*}
    G_t &= \{C_t\boldsymbol{v}_i\mid \sqrt{\lambda_i} \le \epsilon\}, \\
    F_t &= \{C_t\boldsymbol{v}_i\mid \sqrt{\lambda_i} > \epsilon\}.
\end{align*}
If $|F_t| = 0$, the algorithm terminates with output $G=G^t$ and $F=F^t$.
}

\begin{rem}
At \texttt{Step 1}, Eq.~(\ref{eq:orthogonalization}) makes the column space of $C_{t}(X)$ orthogonal to that of $F^{t-1}(X)$. 
The aim is to focus on the subspace of $\mathbb{R}^{|X|}$ that cannot be spanned by the evaluation vectors of polynomials of degree less than $t$. 
\end{rem}

\begin{rem}
At \texttt{Step 3}, a polynomial $C_t\boldsymbol{v}_i$ is classified as an $\epsilon$-vanishing polynomial if $\sqrt{\lambda_i}\le \epsilon$ because $\sqrt{\lambda_i}$ equals the extent of vanishing of $C_t\boldsymbol{v}_i$. Actually,
\begin{align*}
    \|(C_{t}\boldsymbol{v}_{i})(X)\|  =\sqrt{\boldsymbol{v}_{i}^{\top}C_{t}(X)^{\top}C_{t}(X)\boldsymbol{v}_{i}}=\sqrt{\lambda_{i}}.
\end{align*}
\end{rem}
At \texttt{Step 2}, we have a normalization matrix  $\mathfrak{N}(C_t)\in\mathbb{R}^{|C_t|\times|C_t|}$ to resolve the spurious vanishing problem~\cite{kera2019spurious}. For the coefficient normalization, the coefficient vector\footnote{The coefficient vector of a polynomial is defined as a vector that lists the coefficients of the monomials of the polynomial. For instance, a degree-3 univariate polynomial $g=1-x+2x^3$ has the coefficient vector $(1,-1,0,2)^{\top}$.} of $c_i$ is denoted by $\mathfrak{n}_{\mathrm{c}}(c_i)$, and the $(i,j)$-th entry of $\mathfrak{N}(C_t)$ is $\mathfrak{n}_{\mathrm{c}}(c_i)^{\top}\mathfrak{n}_{\mathrm{c}}(c_j)$. Solving the generalized eigenvalue problem Eq.~(\ref{eq:gep}) with this normalization matrix leads to polynomials $C_t\boldsymbol{v}_1,...,C_t\boldsymbol{v}_{|C_t|}$ at \texttt{Step 3}, which are normalized with respect to their coefficient vectors, i.e., $\forall i,\|\mathfrak{n}_{\mathrm{c}}(C_t\boldsymbol{v}_i)\|=1$. Instead of $\mathfrak{n}_{\mathrm{c}}$, we can also define the normalization matrix from another mapping as long as it satisfies some requirements~\cite{kera2019spurious}. Later, we propose a novel mapping $\mathfrak{n}_{\mathrm{g}}$, which is based on the gradient of a given polynomial. Although this mapping only satisfy the relaxed version of the requirements, we show that the same guarantee for the SBC output (Theorem 2 in \citeauthor{kera2019spurious}~\citeyear{kera2019spurious}) can still be stated with these relaxed requirements~(see the supplementary material for proof).
\begin{thm}\label{thm:basis}\label{THM:BASIS}
Let $\mathfrak{n}$ be a valid normalization mapping for SBC (cf., Definition~\ref{def:normalization-mapping}).
When SBC with $\mathfrak{n}$ runs with $\epsilon=0$ for a set of points $X$, the output basis sets $G$ and $F$ satisfy the following. \begin{itemize}
    \item Any vanishing polynomial $g\in\mathcal{I}(X)$ can be generated by $G$, i.e., $g\in\langle G\rangle$. 
    \item Any polynomial $h$ can be represented by $h = f^{\prime} + g^{\prime}$, where $f^{\prime}\in\mathrm{span}(F)$ and $g^{\prime}\in\langle G\rangle$.
    \item For any $t$, any degree-$t$ vanishing polynomial $g\in\mathcal{I}(X)$ can be generated by $G^t$, i.e., $g\in\langle G^t\rangle$. 
    \item For any $t$, any degree-$t$ polynomial $h$ can be represented by $h = f^{\prime} + g^{\prime}$, where $f^{\prime}\in\mathrm{span}(F^t)$ and $g^{\prime}\in\langle G^t\rangle$.
\end{itemize}
\end{thm}

\section{Proposed Method}
In the literature on the vanishing ideal, polynomials are represented by their evaluation vectors at an input set of points $X$. However, two vanishing polynomials, say $g_1$ and $g_2$, share identical evaluation vectors $g_1(X)=g_2(X)=\boldsymbol{0}$, and thus any information about their symbolic forms cannot be inferred from these vectors. 
In this paper, we propose to use the gradient as a key tool to deal with polynomials in a fully numerical way. Specifically, given a polynomial $h$, we consider the evaluation of its partial derivatives $\nabla h := \{\partial h/\partial x_1,\partial h/\partial x_2,\cdots,\partial h/\partial x_n\}$ at the given set of data points $X \subset\mathbb{R}^n$; that is, from the definition of the evaluation matrix, we consider
\begin{align*}
    \nabla h(X) & =\left(\begin{array}{cccc}
    \frac{\partial h}{\partial x_1}(X) & 
    \frac{\partial h}{\partial x_2}(X) & 
    \cdots & \frac{\partial h}{\partial x_n}(X)
    \end{array}\right),
\end{align*}
which can be efficiently and exactly calculated without differentiation by taking advantage of the iterative framework of the basis construction. Interestingly, one can infer the symbolic structure of a vanishing polynomial $g$ from $\nabla g(X)$. For example, if $(\partial g/\partial x_k)(X) \approx \boldsymbol{0}$, then the variable $x_k$ is unlikely to be dominant in $g$; if $(\partial g/\partial x_k)(X)\approx\boldsymbol{0}$ for all $k$, then $g$ can be close to the zero polynomial. One may argue that a nonzero vanishing polynomial $g$ can take $\nabla g(X)=O$. However, such $g$ is revealed to be \textit{redundant} in the basis set, and thus it can be excluded from our consideration (cf., Lemmas~\ref{lem:multiplicity-G} and~\ref{lem:multiplicity-F}). Next, we ask whether any symbolic relation between vanishing polynomials $g_1$ and $g_2$ is reflected in the relation between $\nabla g_1(X)$ and $\nabla g_2(X)$. The answer is yes; if $g_2$ is a polynomial multiple of $g_1$, i.e., $g_2 = g_1h$ for some $h\in\mathcal{P}_n$, then for any $\boldsymbol{x}\in X$, $\nabla g_1(\boldsymbol{x})$ and $\nabla g_2(\boldsymbol{x})$ are identical up to a constant scale. A more general symbolic relation between polynomials is discussed in Conjecture~\ref{conj:multiplicity-theorem}. The proofs of our claims are provided in the supplementary material for reasons of  space.

\subsection{Gradient normalization for the spurious vanishing problem}\label{gradient-based-normalization}
The spurious vanishing problem is resolved by normalizing polynomials for some scale. Here, we propose gradient normalization, which normalizes polynomials using the norm of their gradient. Specifically, a polynomial $h$ is normalized with the norm of the vector
\begin{align}
    \mathfrak{n}_{\mathrm{g}}(h;X) &= \mathrm{vec}(\nabla h(X)) \in \mathbb{R}^{|X|n},\label{eq:grad-mapping}
\end{align}
where $\mathrm{vec}(\cdot)$ denotes the vectorization of a given matrix. We refer to the norm $\|\mathfrak{n}_{\mathrm{g}}(h;X)\|$ as the gradient norm of $h$. By solving Eq.~(\ref{eq:gep}) in \texttt{Step 2}, the basis polynomials of vanishing polynomials and nonvanishing polynomials (say, $h$) are normalized such that $\|\mathfrak{n}_{\mathrm{g}}(h;X)\|=1$. Conceptually, this rescales $h$ with respect to the gradient norm as $h/\|\mathfrak{n}_{\mathrm{g}}(h;X)\|$, but in an optimal way~\cite{kera2019spurious}. 
The gradient normalization is superior to the coefficient normalization in terms of computational cost; the former works in polynomial time complexity~(cf., Proposition~\ref{prop:complexity-of-normalization}) and the latter requires exponential time complexity. SBC using $\mathfrak{n}_{\mathrm{g}}$ (SBC-$\mathfrak{n}_{\mathrm{g}}$) set $m$ of $F_0=\{m\}$ to the mean absolute value of $X$ for consistency.

The gradient normalization is based on a shift in thinking on ``being close to the zero polynomial''. Traditionally, the closeness was measured based on the coefficients---polynomials with small coefficients are considered close to the zero polynomial. On the other hand, the gradient normalization is based on the gradient norm; that is, if the gradient of a polynomial has a small norm at all the given points, then the polynomial is considered close to the zero polynomial. 

A natural concern about the gradient normalization is that the gradient norm $\|\mathfrak{n}_{\mathrm{g}}(h;X)\|$ can be equal to zero even for a nonzero polynomial $h$. In other words, what if all partial derivatives $\partial h/\partial x_k$ are vanishing for $X$, i.e., $(\partial h/\partial x_k)(X)=\boldsymbol{0}$? Solving the generalized eigenvalue problem Eq.~(\ref{eq:gep}) only provides polynomials with the nonzero gradient norm. Is it sufficient for basis construction to only collect such polynomials? The following two lemmas answer this question affirmatively.

\begin{lem}\label{lem:multiplicity-G}
Suppose that $G^t \subset \mathcal{P}_n$ is a basis set of vanishing polynomials of degree at most $t$ for a set of points $X$ such that for any $\widetilde{g}\in\mathcal{I}(X)$ of degree at most $t$, $\widetilde{g}\in\langle G^t\rangle$.
Then, for any $g\in\mathcal{I}(X)$ of degree $t+1$, if $(\partial g/\partial x_k)(X)=\boldsymbol{0}$ for all $k=1,2,...,n$, then $g \in \langle G^t \rangle$.
\end{lem}
\begin{lem}\label{lem:multiplicity-F}
Suppose that $F^t\subset\mathcal{P}_n$ is a basis set of nonvanishing polynomials of degree at most $t$ for a set of points $X$ such that for the evaluation vector $\widetilde{f}(X)$ of any nonvanishing polynomial $\widetilde{f}$ of degree at most $t$, $\widetilde{f}(X)\in\mathrm{span}(F^t(X))$.
Then, for any nonvanishing polynomial $f\in\mathcal{P}_n$ of degree $t+1$, if $(\partial f/\partial x_k)(X)=0$ for all $k=1,2,...,n$, then $f(X) \in \mathrm{span}(F^t(X))$.
\end{lem}
These two lemmas imply that we do not need polynomials with zero gradient norms for constructing basis sets because these polynomials can be described by basis polynomials of lower degrees. Therefore, it is valid to use $\mathfrak{n}_{\mathrm{g}}$ for the normalization in SBC.
Formally, we define the validity of the normalization mapping for a basis construction as follows.
\begin{defn}[Valid normalization mapping for $\mathcal{A}$] Let \label{def:normalization-mapping}
$\mathfrak{n}:\mathcal{P}_n\to \mathbb{R}^{\ell}$ be a mapping that satisfies the following.
\begin{itemize}
    \item $\mathfrak{n}$ is a linear mapping, i.e., $\mathfrak{n}(ah_1+bh_2) = a\mathfrak{n}(h_1)+b\mathfrak{n}(h_2)$, for any $a,b\in\mathbb{R}$ and any $h_1,h_2\in\mathcal{P}_n$.
    \item The dot product is defined between normalization components; that is, $\langle\mathfrak{n}(h_1),\mathfrak{n}(h_2) \rangle$ is defined for any $h_1,h_2\in\mathcal{P}_n$.
    \item In a basis construction algorithm $\mathcal{A}$, $\mathfrak{n}(h)$ takes the zero value only for polynomials that can be generated by basis polynomials of lower degrees.
\end{itemize}
Then, $\mathfrak{n}$ is a valid normalization mapping for $\mathcal{A}$, and $\mathfrak{n}(h)$ is called the normalization component of $h$.
\end{defn}
As the third condition implies, this definition is dependent on the algorithm $\mathcal{A}$. The third condition is the relaxed condition of that in ~\cite{kera2019spurious}, where $\mathfrak{n}(h)$ is required to take a zero value if and only if $h$ is the zero polynomial. 
Now, we can readily show that $\mathfrak{n}_{\mathrm{g}}$ is a valid normalization mapping for SBC.
\begin{thm}
The mapping $\mathfrak{n}_{\mathrm{g}}$ of Eq.~(\ref{eq:grad-mapping}) is a valid normalization mapping for SBC.
\end{thm}
We emphasize that gradient normalization is essentially different from coefficient normalization because it is a data-dependent normalization. The following proposition holds thanks to this data-dependent nature, which argues for consistency in the output of SBC-$\mathfrak{n}_{\mathrm{g}}$ with respect to a translation or scaling of the input data points.
\begin{prop}\label{prop:invariance}
Suppose SBC-$\mathfrak{n}_{\mathrm{g}}$ outputs $(G,F)$ for input $(X,\epsilon)$, $(\widetilde{G},\widetilde{F})$ for input $(X-\boldsymbol{b},\epsilon)$, and $(\widehat{G},\widehat{F})$ for input $(\alpha X,|\alpha|\epsilon)$, where $X-\boldsymbol{b}$ denotes the translation of each point in $X$ by $\boldsymbol{b}$ and $\alpha X$ denotes the scaling by $\alpha\ne 0$.
\begin{itemize}
    \item $G,\widetilde{G}$, and $\widehat{G}$ have exactly the same number of basis polynomials at each degree.
    \item $F,\widetilde{F}$, and $\widehat{F}$ have exactly the same number of basis polynomials at each degree.
    \item Any pair of the corresponding polynomials $\widetilde{h}\in \widetilde{G}\cup \widetilde{F}$ and $h\in G\cup F$ satisfies $h(x_1,x_2,...,x_n)=\widetilde{h}(x_1+b_1,x_1+b_2,...,x_n+b_n)$, where $h(x_1,x_2,...,x_n)$ here denotes a polynomial in $n$ variables $x_1,x_2,...,x_n$ and $\boldsymbol{b}=(b_1,b_2,...,b_n)^{\top}$.
    \item For any pair of the corresponding polynomials $\widehat{h}\in \widehat{G}\cup \widehat{F}$ and $h\in G\cup F$, $\widehat{h}$ is the $(1,\alpha)$-degree-wise identical\footnote{The gist of this property will be explained soon. The definition can be found in the supplementary material.} to $h$.
\end{itemize}
\end{prop}
The first two statements of Proposition~\ref{prop:invariance} argue that translation and scaling on the input points do not affect the inferred dimensionality of the algebraic set where the noisy data approximately lie; an algebraic variety should be described by the same number of polynomials of the same nonlinearity before and after these data transformations. Although this intuition seems natural, to our knowledge, no existing basis construction algorithms have this property. The third statement of Proposition~\ref{prop:invariance} argues that basis polynomials for translated data are polynomials with a variable translation from those of the untranslated data. 
Note that it is not trivial that the \textit{algorithm} outputs these translated polynomials. 
VCA has this translation-invariance property, 
whereas most other basis construction algorithms, including SBC with the coefficient normalization, do not. The last statement of Proposition~\ref{prop:invariance} is the most interesting property and is not held by any other basis construction algorithms, to our knowledge. The $(1,\alpha)$-degree-wise identicality between the corresponding $\widehat{h}\in\widehat{G}\cup \widehat{F}$ and $h\in G\cup F$ implies the following relation:
\begin{align}\label{eq:linear-response}
    \widehat{h}(\alpha X) &= \alpha h(X).
\end{align}
In words, scaling by $\alpha$ on input $X$ of SBC-$\mathfrak{n}_{\mathrm{g}}$ only affects \textit{linearly} the evaluation vectors of the \textit{nonlinear} output polynomials. Thus, we only need linearly scaled threshold $\alpha\epsilon$ for $\alpha X$. Without this property, linear scaling on the input leads to nonlinear scaling on the evaluation of the output polynomials; thus, a consistent result cannot be obtained regardless of how well $\epsilon$ is chosen. Symbolically, $(1,\alpha)$-degree-wise identicality implies that $h$ and $\widehat{h}$ consist of the same terms up to a scale, and the corresponding terms $m$ of $h$ and $\widehat{m}$ of $\widehat{h}$ relate as $\widehat{m} = {\alpha}^{1-\tau}m$. This implies that larger $\alpha$ decreases the coefficients of higher-degree terms more sharply. This is quite natural because highly nonlinear terms grow sharply as the input value increases. 
One may argue that any basis construction algorithm could obtain translation- and scale-invariance by introducing a preprocessing stage for input $X$, such as mean-centralization and normalization. Although preprocessing can be helpful in some practical scenarios, it discards the mean and scale information, and thus the output basis sets do not reflect this information. 
In contrast, the output polynomials of SBC-$\mathfrak{n}_{\mathrm{g}}$ reflect the mean and scale, but in a convenient form. 

\subsection{Removal of redundant basis polynomials}
The monomial-order-free algorithms tend to output a large basis set of vanishing polynomials that contains redundant basis polynomials. Specifically, let $G$ be an output basis set of vanishing polynomials ($\epsilon=0$). Then, $G$ can contain redundant polynomials (say, $g\in G$) that can be generated from polynomials of lower degrees in $G$; that is, with some polynomials $\{h_{g^{\prime}}\}\subset \mathcal{P}_n$, 
\begin{align}\label{eq:redundance}
    g = \sum_{g^{\prime}\in G^{\mathrm{deg}(g)-1}} h_{g^{\prime}}g^{\prime},
\end{align}
which is equivalent to $g\in \langle G^{\mathrm{deg}(g)-1}\rangle$. To determine whether $g\in \langle G^{\mathrm{deg}(g)-1}\rangle$ or not for a given $g$, a standard approach in computer algebra is to divide $g$ by the Gr\"obner basis of $G^{\mathrm{deg}(g)-1}$. However, the complexity of computing a Gr\"obner basis is known to be doubly exponential~\cite{cox1992ideals}. Polynomial division also needs an expanded form of $g$, which is also computationally costly to obtain. Moreover, this polynomial division-based approach is not suitable for the approximate setting, where $g$ may be approximately generated by polynomials in $G^{\mathrm{deg}(g)-1}$. Thus, we would like to handle the redundancy in a numerical way using the evaluation values at points. However, (exact) vanishing polynomials have the same evaluation vectors $\boldsymbol{0}$. 

Here again, we can resort to the gradient of the polynomials, whose evaluation values are proven to be nonvanishing at input points (Lemma~\ref{lem:multiplicity-G}). In short, we consider $g$ as redundant if for any point $\boldsymbol{x}\in X$, the gradient $\nabla g(\boldsymbol{x})$ is linearly dependent on that of the polynomials in $G^{\mathrm{deg}(g)-1}$.

\begin{conj}\label{conj:multiplicity-theorem}
Let $G$ be a basis set of a vanishing ideal $\mathcal{I}(X)$, which is output by SBC with $\epsilon=0$. Then, $g\in G$ is $g\in \langle G^{\mathrm{deg}(g)-1}\rangle$ if and only if for any $\boldsymbol{x}\in X$, 
\begin{align}
    \nabla g(\boldsymbol{x}) &= \sum_{g^{\prime}\in  G^{\mathrm{deg}(g)-1}} \alpha_{g^{\prime},\boldsymbol{x}} \nabla g^{\prime}(\boldsymbol{x}),\label{eq:span-by-other-gradient}
\end{align}
for some $\alpha_{g^{\prime},\boldsymbol{x}}\in\mathbb{R}$.
\end{conj}
The sufficient condition (``if" statement) can be readily proven by differentiating $g=\sum_{g^{\prime}\in G^{\mathrm{deg}(g)-1}}g^{\prime}h_{g^{\prime}}$ and using $g^{\prime}(\boldsymbol{x})=0$ (see the supplementary material).
Using the sufficiency, we can remove all the redundant polynomials in the form of Eq.~(\ref{eq:redundance}) from the basis set by checking whether or not Eq.~(\ref{eq:span-by-other-gradient}) holds. Note that we may accidentally remove some basis polynomials that are not redundant because the necessity (``only if" statement) remains to be proven.  Conceptually, the necessity implies that one can know the global (symbolic) relation $g\in \langle G^{\mathrm{deg}(g)-1}\rangle$ from the local relation Eq.~(\ref{eq:span-by-other-gradient}) at finitely many points $X$. This may not be true for general $g$ and $G^{\mathrm{deg}(g)-1}$. However, $g$ and $G^{\mathrm{deg}(g)-1}$ are both generated in a very restrictive way, and this is why we suspect that this conjecture can be true.

We can support the validity of using Conjecture~\ref{conj:multiplicity-theorem} from another perspective. When Eq.~(\ref{eq:span-by-other-gradient}) holds, this implies the following: using the basis polynomials of lower degrees, one can generate a polynomial $\widehat{g}$ that takes the same value and gradient as $g$ at all the given points; in short, $\widehat{g}$ behaves identically to $g$ up to the first order for all the points. According to the spirit of the vanishing ideal---identifying a polynomial only by its behavior for given points---it is reasonable to consider $g$ as ``redundant" for practical use. 

Lastly, we describe how to use Conjecture~\ref{conj:multiplicity-theorem} to remove redundant polynomials. Given $g$ and $G^{\mathrm{deg}(g)-1}$, we solve the following least squares problem for each $\boldsymbol{x}\in X$: 
\begin{align}
    \min _{\boldsymbol{v}\in\mathbb{R}^{|G^{\mathrm{deg}(g)-1}|}}\|\nabla g(\boldsymbol{x}) - \boldsymbol{v}^{\top}\nabla G^{\mathrm{deg}(g)-1}(\boldsymbol{x}) \|,
\end{align}
where $\nabla G^{\mathrm{deg}(g)-1}(\boldsymbol{x}) \in \mathbb{R}^{|G^{\mathrm{deg}(g)-1}|\times n}$ is a matrix that stacks $\nabla g^{\prime}(\boldsymbol{x})$ for $g^{\prime}\in G^{\mathrm{deg}(g)-1}$ in each row (note that $\nabla g(\boldsymbol{x})\in\mathbb{R}^{1\times n}$). This problem has a closed-form solution $\boldsymbol{v}^{\top}=\nabla g(\boldsymbol{x})\nabla G^{\mathrm{deg}(g)-1}(\boldsymbol{x})^{\dagger}$. If the residual error is zero for all the points in $X$, then $g$ is removed as a redundant polynomial. In the approximately vanishing case ($\epsilon > 0$), we set a threshold for the residual error. The procedure above can be performed during or after basis construction. When the basis construction is not normalized using $\mathfrak{n}_{\mathrm{g}}$, it is also necessary to check the linear dependency of the gradient within $G_t$ (see the supplementary material for details). 

\subsection{Compute the gradient without differentiation}
In our setting, exact gradients for input points can be computed without differentiation. 
Recall that at degree $t$, \texttt{Step 3} of SBC computes linear combinations of the candidate polynomials in $C_{t}$. Noting that $C_{t}$ is generated from the linear combinations of $C_{t}^{\mathrm{pre}}$ and $F^{t-1}$, any $h\in\mathrm{span}(C_{t})$ can be described as 
\begin{align*}
    h &= \sum_{c\in C_{t}^{\mathrm{pre}}}u_c c + \sum_{f\in F^{t-1}}v_f f,
\end{align*}
where $u_c,v_f\in\mathbb{R}$.
Note that $c\in C_{t}^{\mathrm{pre}}$ is a product of a polynomial in $F_1$ and a polynomial in $F_{t-1}$. Let $p_c\in F_1$ and $q_c\in F_{t-1}$ be such polynomials, i.e., $c=p_cq_c$. Using the product rule, the evaluation of $\partial h/\partial x_k$ for $\boldsymbol{x}\in X$ is then
% \begin{align}
%     \frac{\partial h}{\partial x_k}(\boldsymbol{x}) 
%     &= \sum_{c\in C_{t}^{\mathrm{pre}}}u_c \frac{\partial (p_cq_c)}{\partial x_k}(\boldsymbol{x})
%     + \sum_{f\in F^{t-1}}v_f \frac{\partial f}{\partial x_k}(\boldsymbol{x}), \nonumber\\
%     &= \sum_{c\in C_{t}^{\mathrm{pre}}}u_c q_c(\boldsymbol{x})\frac{\partial p_c}{\partial x_k}(\boldsymbol{x}) + \sum_{c\in C_{t}^{\mathrm{pre}}}u_c p_c(\boldsymbol{x})\frac{\partial q_c}{\partial x_k}(\boldsymbol{x})\nonumber\\
%     &\quad\quad + \sum_{f\in F^{t-1}}v_f \frac{\partial f}{\partial x_k}(\boldsymbol{x}).\label{eq:derivative-chain}
% \end{align}
\begin{align}
    \frac{\partial h}{\partial x_k}(\boldsymbol{x}) 
    &= \sum_{c\in C_{t}^{\mathrm{pre}}}u_c q_c(\boldsymbol{x})\frac{\partial p_c}{\partial x_k}(\boldsymbol{x}) + \sum_{c\in C_{t}^{\mathrm{pre}}}u_c p_c(\boldsymbol{x})\frac{\partial q_c}{\partial x_k}(\boldsymbol{x})\nonumber\\
    &\quad\quad + \sum_{f\in F^{t-1}}v_f \frac{\partial f}{\partial x_k}(\boldsymbol{x}).\label{eq:derivative-chain}
\end{align}
Note that $p_c(\boldsymbol{x})$, $q_c(\boldsymbol{x})$, $(\partial p_c/\partial x_k)(\boldsymbol{x})$, $(\partial q_c/\partial x_k)(\boldsymbol{x})$, and $(\partial f/\partial x_k)(\boldsymbol{x})$ have already been calculated in the previous iterations up to degree $t-1$. For degree $t=1$, the gradients of the linear polynomials are the combination vectors $\boldsymbol{v}_i$ obtained in \texttt{Step 2}. Thus, $\nabla h(X)$ can be exactly calculated without differentiation using the results at lower degrees. 

\begin{prop}\label{prop:complexity-of-normalization}
Suppose we perform SBC for a set of points $X\in\mathbb{R}^n$. At the iteration for degree $t$, 
for any polynomial $h \in \mathrm{span}(C_{t}\cup F^{t-1})$ and any point $\boldsymbol{x}\in\mathbb{R}^{n}$, we can compute $\nabla h(\boldsymbol{x})$ without differentiation with a computational cost of $O(n|C_{t}|)=O(n\mathrm{rank}(X)|X|)$.
\end{prop}
This computational cost $O(n\mathrm{rank}(X)|X|)$ is quite acceptable, noting that generating $C_t$ already needs $O(\mathrm{rank}(X)|X|)$ and solving Eq.~(\ref{eq:gep}) needs $O(|C_t|^3)=O(\mathrm{rank}(X)^3|X|^3)$.
Moreover, in this analysis, we use a very rough relation $O(|F_t|)=|X|$, whereas $|F_t|\ll |X|$ in practice (see the supplementary material).
Giving up the exact calculation, one can further reduce the runtime by restricting the variables and points to be taken into account. That is, a normalized component of a polynomial $h$ can be $\widehat{\mathfrak{n}}_{\mathrm{g}}(h) = \nabla_{\Omega}h(Y)$, where $\Omega\subset \{1,2,...,n\}$, $Y \subset X$, and $\nabla_{\Omega}h=\{\partial h/\partial x_i\mid i\in\Omega\}$. For example, $\Omega$ can be the index set of variables that have large variance and $Y$ as the centroids of clusters on $X$.

\section{Results}
We compare four basis construction algorithms, VCA, SBC with the coefficient normalization (SBC-$\mathfrak{n}_{\mathrm{c}}$), SBC-$\mathfrak{n}_{\mathrm{g}}$, and SBC-$\mathfrak{n}_{\mathrm{c}}$ with the basis reduction. All experiments were performed using Julia implementations on a desktop machine with an eight-core processor and 32 GB memory.
\begin{figure*}
\includegraphics[scale=0.268]{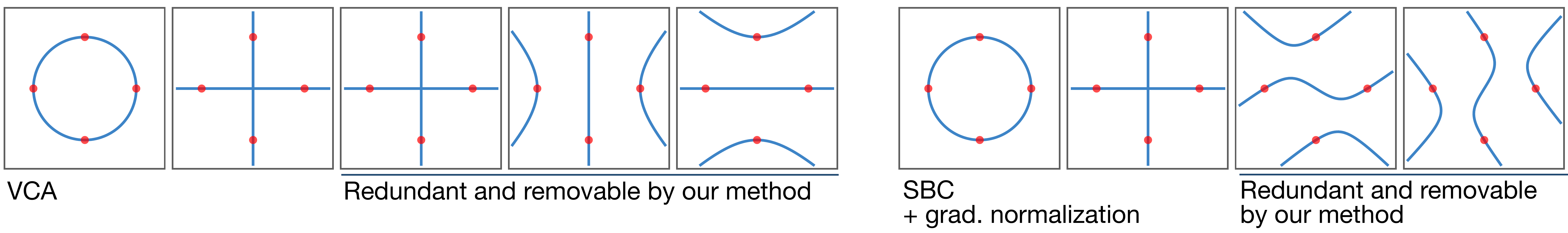}\caption{ Sets of vanishing polynomials obtained by VCA (left panel) and by SBC-$\mathfrak{n}_{\mathrm{g}}$ (right panel). Both sets contain redundant basis polynomials (the last three in the left panel and the last two in the right panel), which can be efficiently removed by the proposed method based on Conjecture~\ref{conj:multiplicity-theorem}.
\label{fig:result-redundancy}}
\end{figure*}

\subsection{Basis reduction using the gradient}
We confirm that redundant basis sets can be reduced by our basis reduction method. We consider the vanishing ideal of $X=\{(1,0),(0,1),(-1,0),(0,-1)\}$ in a noise-free setting, where the exact Gr\"obner basis and polynomial division can be computed to verify our reduction. 
As shown in Fig.~\ref{fig:result-redundancy}, the VCA basis set consists of five vanishing polynomials and the SBC-$\mathfrak{n}_{\mathrm{g}}$ basis set consists of four vanishing polynomials. These basis sets share two polynomials, $g_1=x^2+y^2-1$  and $g_2=xy$ (the constant scale is ignored). A simple calculation using the Gr\"obner basis of $\{g_1,g_2\}$ reveals that the other polynomials in each basis set can be generated by $\{g_1,g_2\}$. Using our  basis reduction method, both basis sets were successfully reduced to $\{g_1,g_2\}$. Other examples and the noisy case can be found in the supplementary material.

\subsection{Comparison of basis sets}
\begin{table}
\caption{Comparison of basis sets obtained by SBC with different normalization ($\mathfrak{n}_{\mathrm{c}}$ and $\mathfrak{n}_{\mathrm{g}}$). Here, $\mathfrak{n}$\textit{-ratio} denotes the ratio of the largest norm to the smallest norm of the polynomials in the basis set with respect to $\mathfrak{n}$. }\label{table:basis-comparison}
\begin{tabular}{|c|c|c|c|c|c|}
\hline 
 &  & \# of bases & $\mathfrak{n}_{\mathrm{c}}$-ratio & $\mathfrak{n}_{\mathrm{g}}$-ratio & runtime (ms)\tabularnewline
\hline 
\hline 
\multirow{2}{*}{D$_1$} & $\mathfrak{n}_{\mathrm{c}}$ & 41 & 1.00  & 12.2e+2 & 48.0e+1\tabularnewline
\cline{2-6} 
 & $\mathfrak{n}_{\mathrm{g}}$ & 30 & 46.6 & 1.00 & \textbf{13.4}\tabularnewline
\hline 
\multirow{2}{*}{D$_2$} & $\mathfrak{n}_{\mathrm{c}}$ & 70 & 1.00 & 19.6e+2 & 17.5e+3\tabularnewline
\cline{2-6} 
 & $\mathfrak{n}_{\mathrm{g}}$ & 33 & 76.9 & 1.00 & \textbf{11.4}\tabularnewline
\hline 
\end{tabular}
\end{table}

We construct two datasets (D$_1$ and D$_2$, respectively) from two algebraic varieties: (i) triple concentric ellipses (radii $(\sqrt{2}, 1/\sqrt{2})$, $(2\sqrt{2}, 2/\sqrt{2})$, and $(3\sqrt{2}, 3/\sqrt{2})$) with $3\pi/4$ rotation and (ii) $\left\{ x_1x_3-x_2^{2},x_1^{3}-x_2x_3\right\}$. 
From each of them, 75 points and 100 points are randomly sampled. Five additional variables $y_i = k_ix_1+(1-k_i)x_2$ for $k_i\in\{0.0, 0.2, 0.5, 0.8, 1.0\}$ are added to the former and nine additional variables $y_i = k_ix_1+l_ix_2+(1-k_i-l_i)x_3$ for $(k_i,l_i)\in \{0.2,0.5,0.8\}^2$ are added to the latter. Then, sampled points are mean-centralized and perturbed by additive Gaussian noise. The mean of the noise is set to zero, and the standard deviation is set to 5\% of the average absolute value of the points. The parameter $\epsilon$ is selected so that (i) the number of linear vanishing polynomials in the basis set agrees with the number of additional variables $y_i$ and (ii) except for these linear polynomials, the lowest degree (say, $d_{\mathrm{min}}$) of the polynomials agree with that of the Gr\"obner basis of the target variety and the number of degree-$d_{\mathrm{min}}$ polynomials in the basis set agrees with or exceeds that of the Gr\"obner basis. Refer to the supplementary material for details. 
As can be seen from Table~\ref{table:basis-comparison}, SBC-$\mathfrak{n}_{\mathrm{g}}$ runs substantially faster than SBC-$\mathfrak{n}_{\mathrm{c}}$ (about 10 times faster in D$_1$ and about $10^3$ times faster in D$_2$). Here, $\mathfrak{n}$-ratio denotes the ratio of the largest to smallest norms of the polynomials in a basis set with respect to $\mathfrak{n}$. Hence, $\mathfrak{n}_{\mathrm{c}}$-ratio and $\mathfrak{n}_{\mathrm{g}}$-ratio are unity for SBC-$\mathfrak{n}_{\mathrm{c}}$ and SBC-$\mathfrak{n}_{\mathrm{g}}$, respectively. 
Here, VCA is not compared because a proper $\epsilon$ could not be found; if the correct number of linear vanishing polynomials were found by VCA, then the  degree-$d_{\mathrm{min}}$ polynomials could not be found, and vice versa. This implies the importance of sidestepping the spurious vanishing problem by normalization.  

\subsection{Classification}
\begin{table}
\caption{Classification results. Here, \textit{dim.} denotes the dimensionality of the extracted features, i.e., the length of $\mathcal{F}(\boldsymbol{x})$, and \textit{br.} denotes the basis reduction. The result of the linear classifier (LC) is shown for reference. The results were averaged over ten independent runs.}\label{table:classification}
\begin{tabular}{|c|c|c|c|c|c||c|}
\hline 
 &  & VCA & \multicolumn{3}{c||}{SBC} & LC\tabularnewline
\hline 
 &  &  & $\mathfrak{n}_{\mathrm{c}}$ & $\mathfrak{n}_{\mathrm{g}}$ & $\mathfrak{n}_{\mathrm{g}}$+br. & \tabularnewline
\hline 
\hline 
\multirow{2}{*}{Iris} & dim. & 80.0 & 44.7 & 148 & \textbf{24.4} & 4\tabularnewline
\cline{2-7} 
 & error & 0.04 & \textbf{0.03} & 0.04 & 0.08 & 0.17\tabularnewline
\hline 
\multirow{2}{*}{Vowel} & dim. & 4744 & 3144 & 3033 & \textbf{254} & 13\tabularnewline
\cline{2-7} 
 & error & 0.44 & \textbf{0.33} & 0.45 & 0.40 & 0.67\tabularnewline
\hline 
\multirow{2}{*}{Vehicle} & dim. & 8205 & 6197 & 5223 & \textbf{260} & 18\tabularnewline
\cline{2-7} 
 & error & 0.18 & 0.22 & \textbf{0.16} & 0.25 & 0.28\tabularnewline
\hline 
\end{tabular}
\end{table}

We compared the basis sets obtained by different basis construction algorithms in the classification tasks. This experiment aims at observing the output of basis construction algorithms for data points not lying on an algebraic variety, and for $\epsilon$ that is tuned for a lower classification error. Following \cite{livni2013vanishing}, the feature vector $\mathcal{F}(\boldsymbol{x})$ of a data point $\boldsymbol{x}$ was defined as 
\begin{align}\label{eq:feature}
\mathcal{F}(\boldsymbol{x}) & =\Bigl(\cdots,\underbrace{\left|g_{1}^{(i)}(\boldsymbol{x})\right|,\cdots,\left|g_{|G_{i}|}^{(i)}(\boldsymbol{x})\right|}_{G_{i}},\cdots\Bigr)^{\top},
\end{align}
where $G_{i}=\{g_{1}^{(i)},...,g_{|G_{i}|}^{(i)}\}$ is the basis set computed for the data points of the $i$-th class. Because of its construction, the $G_i$ part of $\mathcal{F}(\boldsymbol{x})$ is expected to take small values if $\boldsymbol{x}$ belongs to the $i$-th class. 
We trained $\ell_2$-regularized logistic regression with a one-versus-the-rest strategy using LIBLINEAR~\cite{rong09liblinear}. 
We used three small standard datasets (Iris, Vowel, and Vehicle) from the UCI dataset repository~\cite{Lichman2013machine}. Parameter $\epsilon$ was selected by 3-fold cross-validation. Because Iris and Vehicle do not have prespecified training and test sets, we randomly split each dataset into a training set (60\%) and test set (40\%), which were mean-centralized and normalized so that the mean norm of data points is equal to one. The result is summarized in Table~\ref{table:classification}. 
Both SBC-$\mathfrak{n}_{\mathrm{c}}$ and SBC-$\mathfrak{n}_{\mathrm{g}}$ achieved a classification error that is comparable or lower than that of VCA with a much lower dimensionality of feature vectors. In particular, the basis reduction drastically reduces the dimensionality of the feature with a slight change in error. Interestingly, the classification error of VCA is mostly comparable with that of other methods despite many spurious vanishing polynomials and redundant polynomials. We consider this is because these polynomials have little effect on the training of a classifier; spurious vanishing polynomials just extend the feature vector with entries that are close to zero, and redundant basis polynomials behaves like a ``copy" of other non-redundant basis polynomials.
It is interesting to construct the feature vector using discriminative information between classes using discriminative basis construction algorithms, e.g., \cite{kiraly2014dual,hou2016discriminative}. One can consider normalization and basis reduction for these algorithms, but this is beyond the scope of this paper.

\section{Conclusion}
In this paper, we proposed to exploit the gradient of polynomials in the monomial-order-free basis construction of the approximate vanishing ideal. The gradient allows us to access some of the symbolic structure of polynomials and symbolic relations between polynomials. As a consequence, we overcome several theoretical issues in existing monomial-order-free algorithms in a numeraical manner. Specifically, the spurious vanishing problem is resolved in polynomial time complexity for the first time; translation and scaling on the input data lead to consistent changes in the basis set while maintaining the same number polynomials of and same nonlinearity; and redundant basis polynomials are removed from the basis set. These results are achieved efficiently because the gradient of the polynomials at input data points can be exactly computed without differentiation. We believe that this work opens up a new path for monomial-order-free algorithms, which have been theoretically difficult to handle. Future work includes proof of Conjecture~\ref{conj:multiplicity-theorem}. It would also be interesting to replace the coefficient normalization of basis construction algorithms in computer algebra with our gradient normalization. 

\section{Acknowledgement}
This work was supported by JSPS KAKENHI Grant Number 17J07510.

\bibliographystyle{aaai}
% \bibliography{main}

\end{document}

\maketitle

This supplementary material describes the details of our method and presents more results of the numerical experiments. Cross-referencing numbers here are prefixed with S (e.g., Eq.~(S1) or Fig.~S1). Numbers without the prefix (e.g., Eq.~(1) or Fig.~1) refer to numbers in the main text.

\section{Proof of Lemmas~1 and 2}
\begin{lem}\label{lem:poly-decomposition-by-derivatives}
Any polynomial $g\in\mathcal{P}_n$ of degree at least one can be represented as 
\begin{align*}
    g = \sum_{k=1}^n h_k\frac{\partial g}{\partial x_k} + r,
\end{align*}
where $h_k, r\in\mathcal{P}_n$ and $\mathrm{deg}(r) < \mathrm{deg}(g)$.
\end{lem}
\begin{proof}
We provide a constructive proof. For simplicity of notation, we use $t$ such that $t+1=\mathrm{deg}(g)$.
Let $\mathrm{deg}_k(g)$ be the degree of $g$ with respect to the $k$-th variable $x_k$. If $\mathrm{deg}_1(g) < t+1$, we set  $h_1=0$ and proceed to $k=2$. Otherwise, we rearrange $g$ according to the degree of $x_1$ as follows.
\begin{align*}
    g = x_1^{t+1}g_1^{(0)} + x_1^{t}g_1^{(1)} + \cdots + g_1^{(t+1)},
\end{align*}
where $g_{1}^{(\tau)}$ denotes an $(n-1)$-variate polynomial of degree at most $\tau$ that does not contain $x_1$, ($\tau=0,1,...,t+1$). Then,
\begin{align*}
    \frac{\partial g}{\partial x_1} = (t+1)x_1^{t}g_1^{(0)} + tx_1^{t-1}g_1^{(1)} + \cdots + 0.
\end{align*}
By setting $h_1 = x_1/(t+1)$, 
\begin{align*}
    g &= h_1\frac{\partial g}{\partial x_1} + \frac{r_1}{t+1},
\end{align*}
where $r_1 = x_1^{t}g_1^{(1)} + 2x_1^{t-1}g_1^{(2)} \cdots + (t+1)g_1^{(t+1)}$. Note that $\mathrm{deg}_1(r_1)\le t$ and $\mathrm{deg}_l(r_1)\le t+1$ for $l\ne 1$. Next, we perform the same procedures for $k=2$ and $r_1$; if $\mathrm{deg}_2(r_1)<t+1$ then set $h_2=0$ and $r_2 = r_1$, and proceeds to $k=3$; otherwise, rearrange $r_1$ according to the degree of $x_2$ as
\begin{align*}
    r_1 = x_2^{t+1}r_2^{(0)} + x_2^{t}r_2^{(1)} + \cdots + r_2^{(t+1)},
\end{align*}
where $r_{2}^{(\tau)}$ denotes an $(n-1)$-variate polynomial of  degree at most $\tau$ that does not contain $x_2$, ($\tau=0,1,...,t+1$). Again, setting $h_2=x_2/(t+1)$, we obtain
\begin{align*}
    g &= h_1\frac{\partial g}{\partial x_1} + h_2\frac{\partial g}{\partial x_2} + r_2,
\end{align*}
where $r_2 = x_2^{t}r_2^{(1)} + 2x_2^{t-1}r_2^{(2)} \cdots + (t+1)r_2^{(t+1)}$. Note that $\mathrm{deg}_1(r_2)\le t$, $\mathrm{deg}_2(r_2)\le t$ and $\mathrm{deg}_l(r_2)\le t+1$ for $l\ne 1,2$. Repeating this procedure until $k=n$, then $r:=r_n$ satisfies $\mathrm{deg}_l(r)\le t$ for all $l$.
\end{proof}

The extended version of Lemma~\ref{lem:poly-decomposition-by-derivatives} can be proven in a similar fashion.
\begin{lem}\label{lem:poly-decomposition-by-derivatives-2}
Any $g\in\mathcal{P}_n$ of degree at least one can be represented as \begin{align*}
    g = \sum_{k=1}^n h_k\frac{\partial g}{\partial x_k} + r,
\end{align*}
where $h_k, r\in\mathcal{P}_n$ and $\mathrm{deg}_k(r) < \mathrm{deg}_k(g)$ for $k=1,2,...,n$. Here, $\mathrm{deg}_k(\cdot)$ denotes the degree of a given polynomial with respect to the $k$-th variable $x_k$.
\end{lem}

\subsection{Proof of Lemma~1}
\begin{proof}
The degree of $g$ and $\partial g/\partial x_k$ are $t+1$ and at most $t$, respectively. $g$ can be represented as
\begin{align*}
    g = \sum_{k=1}^n h_k\frac{\partial g}{\partial x_k} + r,
\end{align*}
where $h_k$ and $r$ are polynomials.
From Lemma~S1, $h_k$ can be selected so that the degree of $r$ is at most $t$.
By evaluating this for $X$, we obtain $r(X) = \boldsymbol{0}$. Since $G^t$ can generate any vanishing polynomial of degree at most $t$, $r \in \langle G^t \rangle$. Also, $\partial g/\partial x_k\in \langle G^t \rangle$ for $k=1,2,...,n$. From the absorption property of the ideal, $g \in \langle G^t\rangle$.
\end{proof}

\subsection{Proof of Lemma~2}
\begin{proof}
Since $f$ and $\partial f/\partial x_k$ are polynomials of degree $t+1$ and degree at most $t$, respectively, there are polynomials $h$ and $r$ such that
\begin{align*}
    f = \sum_{k=1}^n h_k\frac{\partial f}{\partial x_k} + r,
\end{align*}
where the degree of $h_k$ and $r$ are polynomials. From Lemma~S1, $h_k$ can be selected so that the degree of $r$ is at most $t$. By evaluating this for $X$, we obtain $r(X) = f(X)$. Since column space of $F^t(X)$ spans evaluation vectors of any polynomial of degree at most $t$, $r(X) \in \mathrm{span}(F^t(X))$. 
\end{proof}

\section{Proof of Proposition~1}
\begin{defn}[$(t,\alpha)$-degree-wise identical]\label{def:identical}
Let $k \ne 0$ and let $t$ be an integer. A polynomial $\widetilde{h}\in\mathcal{P}_n$ is $(t,\alpha)$-degree-wise identical to a polynomial $h\in\mathcal{P}_n$ if $h$ and $\widetilde{h}$ consist of the same terms up to scale, and any pair of the corresponding terms $m$ of $h$ and $\widetilde{m}$ of $\widetilde{h}$ satisfies $\widetilde{m} = {\alpha}^{t-\mathrm{deg}(\widetilde{m})}m$.
\end{defn}
For instance, $\widetilde{h} = x^2y+8y$ is $(3,2)$-degree-wise identical to $h = x^2y + 2y$. 

\begin{lem}\label{lem:lin-comb identical}
Consider sets of polynomials, $H = \{h_1,h_2,...,h_s\}\subset\mathcal{P}_n$ and $\widetilde{H} = \{\widetilde{h}_1,\widetilde{h}_2,...,\widetilde{h}_s\}\subset\mathcal{P}_n$, where $\widetilde{h}_i$ is $(t,\alpha)$-degree-wise identical to $h$ for $i=1,2,...,s$. Then, any nonzero vectors $\widetilde{\boldsymbol{w}},\boldsymbol{w}\in\mathbb{R}^s$ such that $\widetilde{\boldsymbol{w}}={\alpha}^{\tau}\boldsymbol{w}$ yields a polynomial $\widetilde{H}\boldsymbol{w}$ that are $(t+\tau,\alpha)$-degree-wise identical to $H\boldsymbol{w}$.
\end{lem}
\begin{proof}
The proof is trivial from Definition~\ref{def:identical}.
\end{proof}

\begin{lem}\label{lem:correspondence}
Let a polynomial $\widetilde{h}\in\mathcal{P}_n$ be $(t,\alpha)$-degree-wise identical to a polynomial $h\in\mathcal{P}_n$. Let $X\subset\mathbb{R}^n$ be a set of points. Then, $\widetilde{h}(\alpha X) = {\alpha}^{t}h(X)$ and $\nabla \widetilde{h}(\alpha X) = {\alpha}^{t-1}\nabla h(X)$, and thus, 
\begin{align*}
    \frac{\widetilde{h}}{\|\nabla\widetilde{h}(\alpha X)\|}(\alpha X) = \alpha \frac{h}{\|\nabla h(X)\|}(X). 
\end{align*}
\end{lem}
\begin{proof}
Let $m$ and $\widetilde{m}$ be any corresponding terms between $h$ and $\widetilde{h}$, which satisfy $\mathrm{deg}(\widetilde{m})=\mathrm{deg}(m)$ and $\widetilde{m} = {\alpha}^{t-\mathrm{deg}(\widetilde{m})}m$. Then,
\begin{align*}
\widetilde{m}(\alpha X) &= {\alpha}^{t-\mathrm{deg}(\widetilde{m})}m(\alpha X), \\ 
&= {\alpha}^{t-\mathrm{deg}(\widetilde{m})}{\alpha}^{\mathrm{deg}(m)}m(X),\\
&= {\alpha}^t m(X).
\end{align*}
Similarly, for any $k$, 
\begin{align*}
\frac{\partial\widetilde{m}}{\partial x_k}(\alpha X) 
&= {\alpha}^{t-\mathrm{deg}(\widetilde{m})}\frac{\partial m}{\partial x_k}(\alpha X), \\ 
&= {\alpha}^{t-\mathrm{deg}(\widetilde{m})} {\alpha}^{\mathrm{deg}(m)-1}\frac{\partial m}{\partial x_k}(X),\\
&= {\alpha}^{t-1} \frac{\partial m}{\partial x_k}(X),
\end{align*}
resulting in $\nabla \widetilde{m}(\alpha X) = {\alpha}^{t-1}\nabla m(X)$.
\end{proof}

\begin{lem}\label{lem:linear-response}
Suppose we perform SBC with the gradient-based normalization for $(S,\epsilon)$ and for $(\alpha X,|\alpha|\epsilon)$, ($k \ne 0$), and obtain $(F^{\tau},G^{\tau})$ and $(\widetilde{F}^{\tau},\widetilde{G}^{\tau})$, respectively, up to degree $t=\tau$.
Suppose that for $t \le \tau$, $F_{t}\cup G_{t}$ and $\widetilde{F}_{t}\cup \widetilde{G}_{t}$ have one-to-one correspondence, say, $h\in F_{t}\cup G_{t}$ and $\widetilde{h}\in\widetilde{F}_{t}\cup \widetilde{G}_{t}$, where $\widetilde{h}$ is degree-wise-$(1,\alpha)$ identical to $h$.
Then, the same claim holds for $t=\tau+1$.
\end{lem}

\begin{proof}
For $t=\tau+1$, $C_{\tau+1}$ and $\widetilde{C}_{\tau+1}$ are generated from $(F_1,F_{\tau})$ and $(\widetilde{F}_1,\widetilde{F}_{\tau+1})$, respectively. From the one-to-one correspondence between $F_t$ and $\widetilde{F}_t$ for $t\le\tau$, There is also one-to-one correspondence between $C_{\tau+1}^{\mathrm{pre}}$ and $\widetilde{C}_{\tau+1}^{\mathrm{pre}}$.
Moreover, $\widetilde{c}^{\mathrm{pre}}\in \widetilde{C}_{\tau+1}^{\mathrm{pre}}$ is degree-wise-($2,k$) identical to the corresponding $c^{\mathrm{pre}}\in C_{\tau+1}^{\mathrm{pre}}$ due to the assumption. Suppose $c^{\mathrm{pre}}$ and $\widetilde{c}^{\mathrm{pre}}$ becomes $c \in C_{\tau+1}$ and $\widetilde{c}\in \widetilde{C}_{\tau+1}$ after the orthogonalization Eq.~(1). By Lemma~\ref{lem:orthogonalization-keeps identicality}, $\widetilde{c}$ is degree-wise-($2,k$) identical to $c$.

From Lemma~\ref{lem:correspondence}, $\widetilde{h}(\alpha X) = {\alpha}^{2}h(X)$ and $\nabla\widetilde{h}(\alpha X) = k\nabla h(X)$. Therefore, $\widetilde{C}_{\tau+1}(\alpha\widetilde{X}) = {\alpha}^{2}C_{\tau+1}(X)$ and $\nabla \widetilde{C}_{\tau}(\alpha\widetilde{X}) = k \nabla C_{\tau+1}(X)$. Note that $\mathfrak{n}_{\mathrm{g}}(\widetilde{C}_{\tau+1})(\alpha X) = {\alpha}^2\mathfrak{n}_{\mathrm{g}}(C_{\tau+1})(X)$.
Thus, the generalized eigenvalue problem Eq.~(2)
\begin{align*}
    \widetilde{C}_{\tau+1}(\alpha X)^{\top}\widetilde{C}_{\tau+1}(\alpha X)\widetilde{V}
    = \mathfrak{N}_{\mathrm{g}}(\widetilde{C}_{\tau+1};\alpha X)
    \widetilde{V}\widetilde{\Lambda},
\end{align*}
where $\mathfrak{N}_{\mathrm{g}}$ is $\mathfrak{N}$ for the gradient-based normalization, is equivalent to 
\begin{align*}
C_{\tau+1}(X)^{\top}C_{\tau+1}(X)\widetilde{V} = \frac{1}{{\alpha}^2}\mathfrak{N}_{\mathrm{g}}(C_{\tau+1};X)\widetilde{V}\widetilde{\Lambda},
\end{align*}
which leads to $\widetilde{\Lambda}={\alpha}^2 \Lambda$. Also, 
\begin{align*}
    \widetilde{V}^{\top}\widetilde{C}_{\tau+1}(\alpha X)^{\top}\widetilde{C}_{\tau+1}(\alpha X)\widetilde{V} &=I,\\
    {\alpha}^2 \widetilde{V}^{\top}C_{\tau+1}( S)^{\top}C_{\tau+1}(X)\widetilde{V} &=I.
\end{align*}
Comparing with $V^{\top}C_{\tau+1}(X)^{\top}C_{\tau+1}(X)V=I$, we obtain $\widetilde{V} = {\alpha}^{-1}V$.

Let $\boldsymbol{v}_i$ and $\widetilde{\boldsymbol{v}}_i$ be the $i$-th column of $V$ and $\widetilde{V}$, respectively. From $\widetilde{\boldsymbol{v}}_i={\alpha}^{-1}\boldsymbol{v}_i$ and Lemma~\ref{lem:lin-comb identical}, $\widetilde{C}_{1}\widetilde{\boldsymbol{v}}_i$ is $(1,\alpha)$-degree-wise identical to $C_{\tau+1}\boldsymbol{v}_i$. Hence, any polynomials $h\in F_{\tau+1}\cup G_{\tau+1}$ and $\widetilde{h}\in \widetilde{F}_{\tau+}\cup \widetilde{G}_{\tau+1}$ satisfy $\widetilde{h}(\alpha X) = \alpha h(X)$. 
This fact is also supported by $\widetilde{\Lambda}={\alpha}^2 \Lambda$ (recall that the square root of the eigenvalues corresponds to the extent of vanishing).
Note that polynomials in $\widetilde{C}_{\tau+1}\widetilde{V}$ are assorted into $\widetilde{F}_{\tau+1}$ or $\widetilde{G}_{\tau+1}$ by the threshold $k \epsilon$, which leads to the same classification as $F_{\tau+1}$ and $G_{\tau+1}$ by $\epsilon$. Thus, the one-to-one correspondence is kept between $F_{\tau+1}$ and $\widetilde{F}_{\tau+1}$ and also between $G_{\tau+1}$ and $\widetilde{G}_{\tau+1}$.
\end{proof}

\begin{lem}\label{lem:orthogonalization-keeps identicality}
Consider the same setting in Lemma~\ref{lem:correspondence}: suppose we perform SBC with the gradient-based normalization for $(S,\epsilon)$ and for $(\alpha X,|\alpha|\epsilon)$, ($k\ne 0$), and obtain $(F^{\tau},G^{\tau})$ and $(\widetilde{F}^{\tau},\widetilde{G}^{\tau})$, respectively, up to degree $t=\tau$.
Suppose that for $t \le \tau$, $F_{t}\cup G_{t}$ and $\widetilde{F}_{t}\cup \widetilde{G}_{t}$ have one-to-one correspondence, say, $h\in F_{t}\cup G_{t}$ and $\widetilde{h}\in\widetilde{F}_{t}\cup \widetilde{G}_{t}$, where $\widetilde{h}$ is degree-wise-$(1,\alpha)$ identical to $h$. Now, suppose $h^{\mathrm{pre}}\in C_{\tau+1}^{\mathrm{pre}}$ and $\widetilde{c}^{\mathrm{pre}}\in\widetilde{C}_{\tau+1}^{\mathrm{pre}}$ becomes $c \in C_{\tau+1}$ and $\widetilde{c}\in \widetilde{C}_{\tau+1}$ after the orthogonalization Eq.~(1). Then, $\widetilde{c}$ is $(2,\alpha)$-degree-wise identical to $c$.
\end{lem}
\begin{proof}
The entry-wise description of the orthogonalization Eq.~(1) for $C_{\tau+1}$ and $\widetilde{C}_{\tau+1}$ is respectively as follows.
\begin{align*}
    c &= c^{\mathrm{pre}} - F^{\tau}F^{\tau}(X)^{\dagger}c^{\mathrm{pre}}(X),\\
    \widetilde{c} &= \widetilde{c}^{\mathrm{pre}} - \widetilde{F}^{\tau}\widetilde{F}^{\tau}(\alpha X)^{\dagger}\widetilde{c}^{\mathrm{pre}}(\alpha X).
\end{align*}
Let $\widetilde{\boldsymbol{w}}=\widetilde{F}^{\tau}(\alpha X)^{\dagger}\widetilde{c}^{\mathrm{pre}}(\alpha X)$ and $\boldsymbol{w}=F^{\tau}(X)^{\dagger}c^{\mathrm{pre}}(X)$. We will now show $\widetilde{\boldsymbol{w}} = k \boldsymbol{w}$. If this holds, each entry of $\widetilde{F}^{\tau}\widetilde{\boldsymbol{w}}$ becomes degree-wise-($2,k$) identical to the corresponding entry of $F^{\tau}\boldsymbol{w}$. Thus, from Lemma~\ref{lem:lin-comb identical}, $\widetilde{c}$ is $(2,\alpha)$-degree-wise identical to $c$.

First, note that the column vectors of $\widetilde{F}^{\tau}(\alpha X)$ are mutually orthogonal by construction because the orthogonalization makes $\mathrm{span}(\widetilde{F}_{t_1}(\alpha X))$ and $\mathrm{span}(\widetilde{F}_{t_2}(\alpha X))$ mutually orthogonal for any $t_1\ne t_2$, and the generalized eigenvalue decomposition makes the columns of $\widetilde{F}_{t}(\alpha X)$ mutually orthogonal for $t\le\tau$. Therefore, 
\begin{align*}
    \widetilde{D} &:= \widetilde{F}^{\tau}(\alpha X)^{\top}\widetilde{F}^{\tau}(\alpha X),\\
    &={\alpha}^{2}F^{\tau}(X)^{\top}F^{\tau}(X),\\
    &=: {\alpha}^{2} D,
\end{align*}
where both $\widetilde{D}$ and $D$ are diagonal matrices with positive entries in their diagonal.
Hence, the pseudo-inverse becomes
\begin{align*}
\widetilde{F}^{\tau}(\alpha X)^{\dagger}&=\widetilde{D}^{-1}\widetilde{F}^{\tau}(\alpha X)^{\top},\\
&= ({\alpha}^{-2}D^{-1})(\alpha F^{\tau}(X)^{\top}),\\
&= {\alpha}^{-1} D^{-1}F^{\tau}(X)^{\top},\\ 
&= {\alpha}^{-1} F^{\tau}(X)^{\dagger}.  
\end{align*}
 Therefore, 
\begin{align*}
\widetilde{\boldsymbol{w}}&=\widetilde{F}^{\tau}(\alpha X)^{\dagger}\widetilde{c}^{\mathrm{pre}}(\alpha X),\\
&= {\alpha}^{-1}F^{\tau}(X)^{\dagger}({\alpha}^2 c^{\mathrm{pre}}(X)),\\
&= \alpha F^{\tau}(X)^{\dagger}c^{\mathrm{pre}}(X),\\
&= \alpha\boldsymbol{w}.
\end{align*}
\end{proof}

\section{Proof of the sufficiency of Conjecture~1}
\begin{proof}
From the assumption $g\in \langle G^{\mathrm{deg}(g)-1}\rangle$, we can represent $g$ as $g=\sum_{g^{\prime}\in G^{\mathrm{deg}(g)-1}}g^{\prime}h_{g^{\prime}}$, for some $\{h_{g^{\prime}}\}\subset \mathcal{P}_n$. Thus, 
\begin{align*}
    \nabla g(\boldsymbol{x}) &= \sum_{g^{\prime}\in G^{\mathrm{deg}(g)-1}} h_{g^{\prime}}(\boldsymbol{x})\nabla g^{\prime}(\boldsymbol{x})
    +g^{\prime}(\boldsymbol{x})\nabla h_{g^{\prime}}(\boldsymbol{x}),\\
    &= \sum_{g^{\prime}\in G^{\mathrm{deg}(g)-1}} h_{g^{\prime}}(\boldsymbol{x})\nabla g^{\prime}(\boldsymbol{x})\label{eq:grad-expansion},
\end{align*}
where we used $g^{\prime}(\boldsymbol{x})=0$ in the last equality. 
\end{proof}

\section{Proof of Proposition~2}
\begin{proof}
Equation~(8) shows that the evaluation of a partial derivative $(\partial h/\partial x_k)(\boldsymbol{x})$ is the sum of $|C_t^{\mathrm{pre}}|+|F^{t-1}|$ terms (note $|C_t|=|C_t^{\mathrm{pre}}|$). Hence, $\nabla h(\boldsymbol{x})$ requires $O(n(|C_t|+|F^{t-1}|))$.
Note that for an final output $F$ of SBC $|F|\le |X|$ holds. This is because $F(X)$ is full-rank thanks to the orthogonalization in Eq.~(1), and $\mathrm{rank}(F^t(X)) \le |X|$ because $\mathrm{span}(F(X))\subseteq \mathrm{R}^{|X|}$ (the equalities hold at $\epsilon=0$), where $\mathrm{rank}(\cdot)$ denotes the matrix rank of given matrix.  
Hence, $O(|F^{t-1}|)=O(|X|)$.
Also, by its construction, $|F_1|\le \mathrm{rank}(X)$.
Therefore, $|C_t|=|F_1||F_{t-1}|\le \mathrm{rank}(X)|X|$, and thus, $O(|C_t|)=O(\mathrm{rank}(X)|X|)$. Thus, $O(n(|C_t|+|F^{t-1}|))=O(n|C_t|)=O(n\cdot \mathrm{rank}(X)|X|)$.
\end{proof}

\section{Proof of Theorem~1}\label{ssec:basis-is-basis}
\paragraph{Structure and notations} The proof relies on Theorem~1 in \cite{kera2019spurious}, which shows that SBC with $\mathfrak{N}(C_t)=I$ satisfies Theorem~1. The structure of the proof here is following the proof of Theorem~2 in \cite{kera2019spurious}, which compares two processes of SBC, SBC with $\mathfrak{N}(C_t)=I$ and SBC with $\mathfrak{n}$ that satisfies Definition~4. We refer to the former and the latter, respectively, as SBC$_I$ and SBC$_{\mathfrak{n}}$.
If we use symbols such as $G_t$ and $F_t$ in SBC$_I$, we put tildes on the corresponding symbols such as $\widetilde{G}_t$ and $\widetilde{F}_t$ in SBC$_{\mathfrak{n}}$.

\begin{proof}
We prove the claim by induction with respect to degree $t$. 
Let $F_t$ and $G_t$ be the basis sets of nonvanishing polynomials and vanishing polynomials, respectively, obtained at degree $t$ iteration in SBC$_I$.
From Theorem~1 in~\cite{kera2019spurious}, we know that collecting $F_t$ and $G_t$ yields complete basis sets for both nonvanishing and vanishing polynomials. Here, we prove the claim by comparing $\widetilde{F}_t$ and $\widetilde{G}_t$ with $F_t$ and $G_t$. Specifically, we show $\mathrm{span}(\widetilde{F}_t)=\mathrm{span}(F_t)$ and $\langle \widetilde{G}^t\rangle = \langle G^t\rangle$. 
Note that $\mathrm{span}(\widetilde{F}_t)\subset\mathrm{span}(F_t)$ and $\langle \widetilde{G}^t\rangle \subset \langle G^t\rangle$ are obvious because $\widetilde{F}_t$ and $\widetilde{G}_t$ are generated by assigning additional constraints of normalization on the generation of $F_t$ and $G_t$.
Thus, our goal is to prove the reverse inclusion 
$\mathrm{span}(\widetilde{F}_t)\supset\mathrm{span}(F_t)$ and $\langle \widetilde{G}^t\rangle \supset \langle G^t\rangle$.

At $t=1$, it is obvious that $\mathrm{span}(F_1)=\mathrm{span}(\widetilde{F}_1)$ and $\langle G^1\rangle=\langle\widetilde{G}^1\rangle$. We assume $\mathrm{span}(F_t) = \mathrm{span}(\widetilde{F}_t)$ and $\langle G^t\rangle = \langle \widetilde{G}^t\rangle$ for all $t\le\tau$. Then, we can show $\mathrm{span}(C_{\tau+1}^{\mathrm{pre}})=\mathrm{span}(\widetilde{C}_{\tau+1}^{\mathrm{pre}})$ and $\mathrm{span}(C_{\tau+1})=\mathrm{span}(\widetilde{C}_{\tau+1})$. In fact, it is $pq\in \mathrm{span}(\widetilde{C}_{\tau+1}^{\mathrm{pre}})$ for any $pq\in C_{\tau+1}^{\mathrm{pre}}$, where $p\in F_1$ and $q\in F_{\tau}$, because $p\in \mathrm{span}(\widetilde{F}_1)$ and $q\in\mathrm{span}(\widetilde{F}_{\tau})$, and vice versa. The orthogonalization Eq.~(1) projects $\mathrm{span}(C_{\tau+1}^{\mathrm{pre}})$ to subspace $\mathrm{span}(C_{\tau+1})$, which are orthogonal to $\mathrm{span}(F^{\tau})$ in terms of the evaluation at points, i.e., $\mathrm{span}(C_{\tau+1}(X)) \perp \mathrm{span}(F^{\tau}(X))$. From $\mathrm{span}(C_{\tau+1}^{\mathrm{pre}})=\mathrm{span}(\widetilde{C}_{\tau+1}^{\mathrm{pre}})$ and $\mathrm{span}(F^{\tau}(X))=\mathrm{span}(\widetilde{F}^{\tau}(X))$, the orthogonalization  projects $C_{\tau+1}^{\mathrm{pre}}$ and $\widetilde{C}_{\tau+1}^{\mathrm{pre}}$ into the same subspace, i.e., $\mathrm{span}(C_{\tau+1})=\mathrm{span}(\widetilde{C}_{\tau+1})$.

Next, we show $\mathrm{span}(F_{\tau+1})=\mathrm{span}(\widetilde{F}_{\tau+1})$ by showing $\mathrm{span}(F_{\tau+1})\subset \mathrm{span}(\widetilde{F}_{\tau+1})$.
 Let $f\in\mathrm{span}(F_{\tau+1})$ be an nonvanishing polynomial. As shown above, $f\in\mathrm{span}(C_{\tau+1}) = \mathrm{span}(\widetilde{C}_{\tau+1})$. By construction, $f\in\mathrm{span}(\widetilde{F}_{\tau+1})$ unless $\mathfrak{n}(f)=\boldsymbol{0}$. On the other hand if $\mathfrak{n}(f)=\boldsymbol{0}$, then $f$ need not be included in basis sets because of the third requirement of $\mathfrak{n}$. 
 Therefore, $\mathrm{span}(F_{\tau+1})=\mathrm{span}(\widetilde{F}_{\tau+1})$.

We show $\langle G^{\tau+1}\rangle = \langle \widetilde{G}^{\tau+1}\rangle$ in a similar way. Let $g\in \mathrm{span}(G_{\tau+1})$ be a vanishing polynomial. Then, $g\in \mathrm{span}(\widetilde{G}_{\tau+1})$ unless $\mathfrak{n}(g) = \boldsymbol{0}$. From the third requirement for $\mathfrak{n}$, if $\mathfrak{n}(g)=0$ implies $g$ need not be included in $G_{\tau+1}$. Therefore, $\langle \widetilde{G}^t\rangle \supset \langle G^t\rangle$ and thus $\langle \widetilde{G}^t\rangle = \langle G^t\rangle$.
\end{proof}

\section{Basis reduction for the basis construction that is not normalized with $\mathfrak{n}_{\mathrm{g}}$}
Let $\mathfrak{n}$ be a valid normalization mappting for SBC and let $G_t$ be the basis set of vanishing polynomials obtained by SBC-$\mathfrak{n}$ at degree $t$. When $\mathfrak{n}=\mathfrak{n}_{\mathrm{g}}$, $\mathfrak{N}(G_t)$ is a full-rank matrix since the its column vectors are unit and mutually orthogonal. On the  other hand, when $\mathfrak{n}\ne \mathfrak{n}_{\mathrm{g}}$, $\mathfrak{N}(G_t)$ may not be a full-rank matrix. In this case, we need an additional procedure in our basis reduction as follows: (i) compute the rank of $\mathfrak{N}(G_t)$ and (ii) remove $|G_t| - \mathrm{rank}(\mathfrak{N}(G_t))$ polynomials from $G_t$ according to the extent of vanishing in ascending order (polynomials with the smallest extent of vanishing is removed first).

\section{Proof of Theorem~2}
\begin{proof}
It is trivial that the first two requirements are satisfied. As for the third requirement, Lemmas~1 and 2 state that we do not need to use polynomials with zero gradient norm for the basis sets of both vanishing polynomials and nonvanishing polynomials. Note that Lemma~2 implies that a nonvanishing polynomial $f$ such that $\mathfrak{n}_{\mathrm{g}}(f)$ are not included in $\mathrm{span}(C_{t+1})$ because the orthogonalization Eq.~(1) forces $\mathrm{span}(F^{t})(X)\perp \mathrm{span}(C_{t+1})(X)$ regardless of the use of $\mathfrak{n}_{\mathrm{g}}$. 
\end{proof}

\section{Detail of experiments}
In the first experiment, $\epsilon$ is selected for a basis construction algorithm $\mathcal{A}$ and a variety as follows.  First, we compute the Gr\"obner basis of the variety with the degree-reverse-lexicographic order, which determines the number $N$ of the lowest degree $d_{\mathrm{min}}$ of basis polynomials. We added $M$ dummy variables to our datasets (five for D$_1$ and nine for D$_2$) and thus $M$ approximate vanishing polynomials of degree 1 should be obtained, too. Using a linear search, we estimate the range $(\epsilon_{1},\epsilon_{2})$ of $\epsilon$ with which $\mathcal{A}$ outputs a basis set of vanishing polynomials such that (i) $M$ linear polynomials are contained, (ii) at least $N$ polynomials of degree $d_{\mathrm{min}}$ are contained, and (iii) no nonlinear polynomials of degree less than $d_{\mathrm{min}}$ are contained. Finally, $\epsilon$ is set to $\epsilon=(\epsilon_1+\epsilon_2)/2$. The condition (ii) does not require "exactly" $N$ but "at least" $N$. This is because VCA and SBC do not calculate the Gr\"obner basis and thus there can be more than $N$ polynomials of degree $d_{\mathrm{m}}$ for several reasons (especially, the spurious vanishing problem and redundancy of the basis set). We set the threshold of the basis reduction to 1e-9.

\subsection{Additional result}
We provide addition results of the first experiment in Figures~\ref{fig:redundancy-4points},~\ref{fig:redundancy-4points-noisy},~\ref{fig:redundancy-6points}, and~\ref{fig:redundancy-6points-noisy}. Basis sets obtained by VCA, SBC-$\mathfrak{n}_{\mathrm{g}}$, and SBC-$\mathfrak{n}_{\mathrm{c}}$ are compared. The basis reduction method is also applied. The Fig.~\ref{fig:redundancy-4points} is the full-version of Fig.~1. It can be seen that the proposed basis reduction method allow us to remove the redundant basis polynomials.
Next, we perturbed the four points by additive Gaussian noise $\mathcal{N}(0,0.05)$, and the obtained basis sets with $\epsilon=0.05$ are shown in Fig.~\ref{fig:redundancy-4points-noisy}. in a noisy case, we only can discuss redundant polynomials as those \textit{approximately} generated by other lower-degree polynomials. Thus, it is no longer meaningful to compute the exact Gr\"obner basis to find the \textit{exact} redundant basis polynomials. A few polynomials of obtained basis sets are suggested as redundant by our method, which seems reasonable according to the noise-free case. The third basis polynomial in VCA basis set is drawn in clutter lines, which do not go through the points. This is because this polynomial is close to the zero function; in fact, the coefficient norm of this polynomial is 1.2e-16. Similar results are obtained for another set of points both in a noise-free case Fig.~\ref{fig:redundancy-6points} and noisy case Fig.~\ref{fig:redundancy-6points-noisy}.

\begin{figure*}
\includegraphics[scale=0.4]{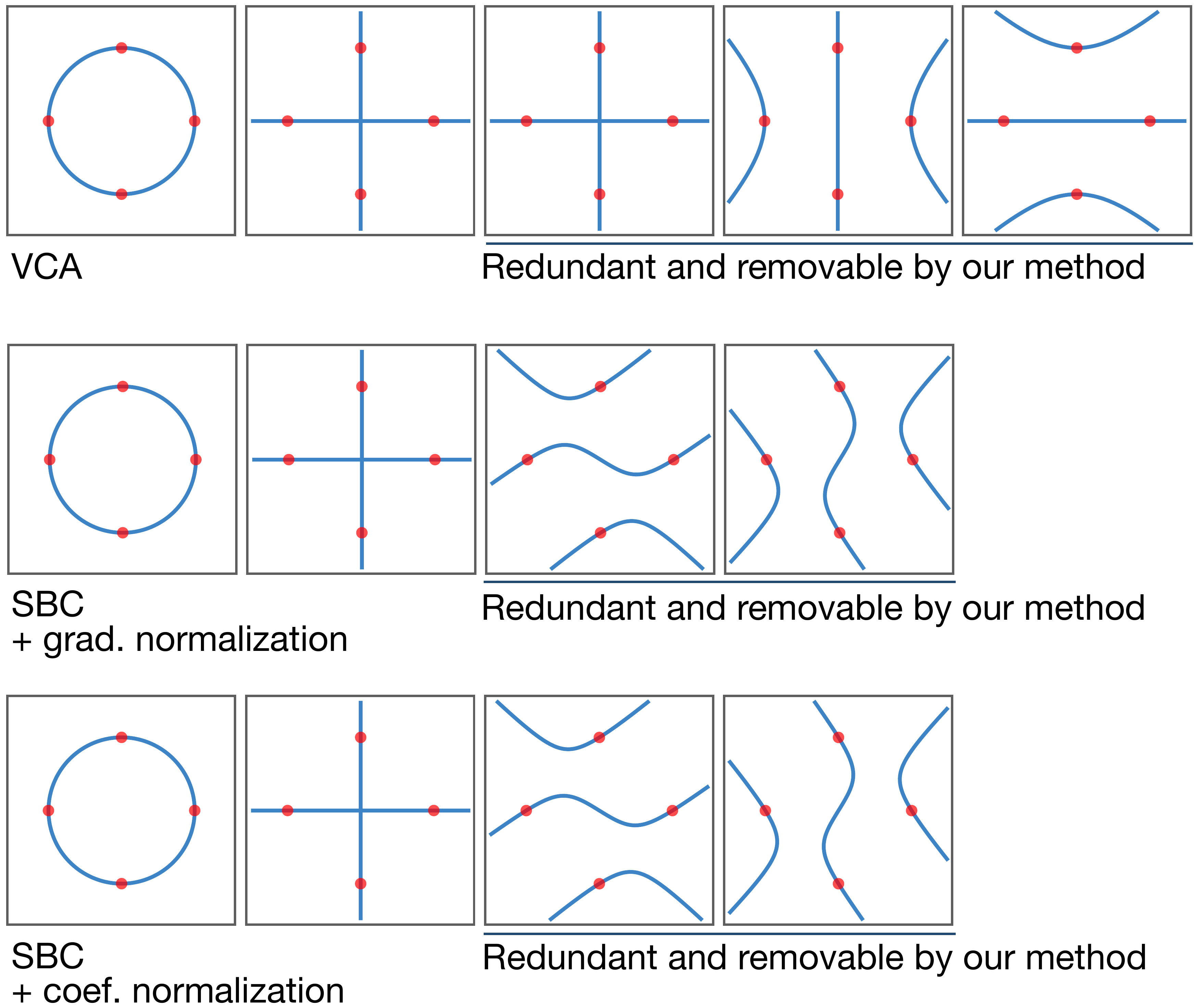}
\caption{Sets of vanishing polynomials obtained by VCA (top row), by SBC-$\mathfrak{n}_{\mathrm{g}}$ (middle row), and by SBC-$\mathfrak{n}_{\mathrm{c}}$ (bottom row) for four-point dataset in a noise-free case. All sets contain redundant basis polynomials, which can be efficiently removed by the basis reduction.}\label{fig:redundancy-4points}
\end{figure*}

\begin{figure*}
\includegraphics[scale=0.4]{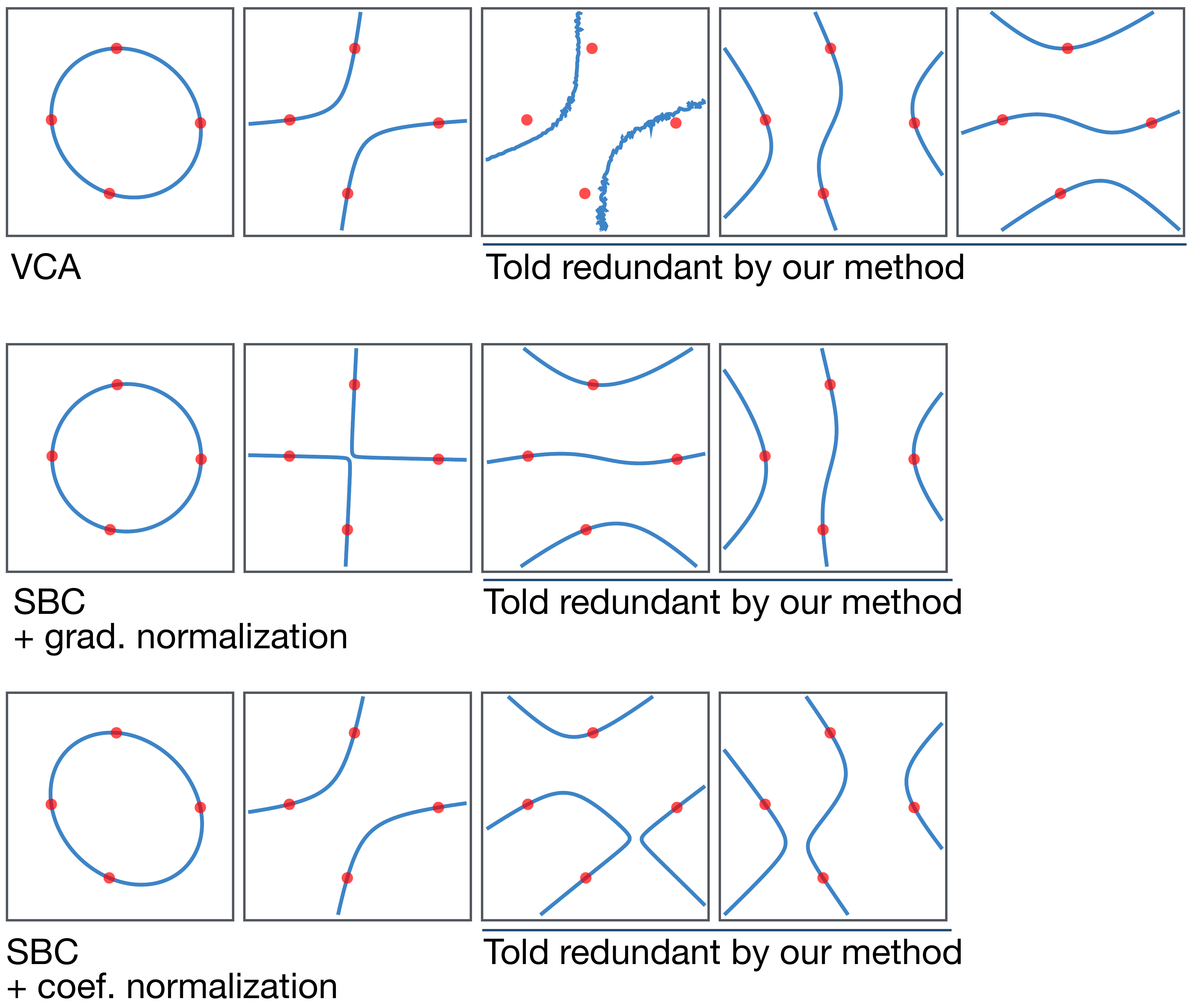}
\caption{Sets of vanishing polynomials obtained by VCA (top row), by SBC-$\mathfrak{n}_{\mathrm{g}}$ (middle row), and by SBC-$\mathfrak{n}_{\mathrm{c}}$ (bottom row) for four-point dataset in a noisy case. All sets contain polynomials that are suggested as redundant by the basis reduction.}\label{fig:redundancy-4points-noisy}
\end{figure*}

\begin{figure*}
\includegraphics[scale=0.28]{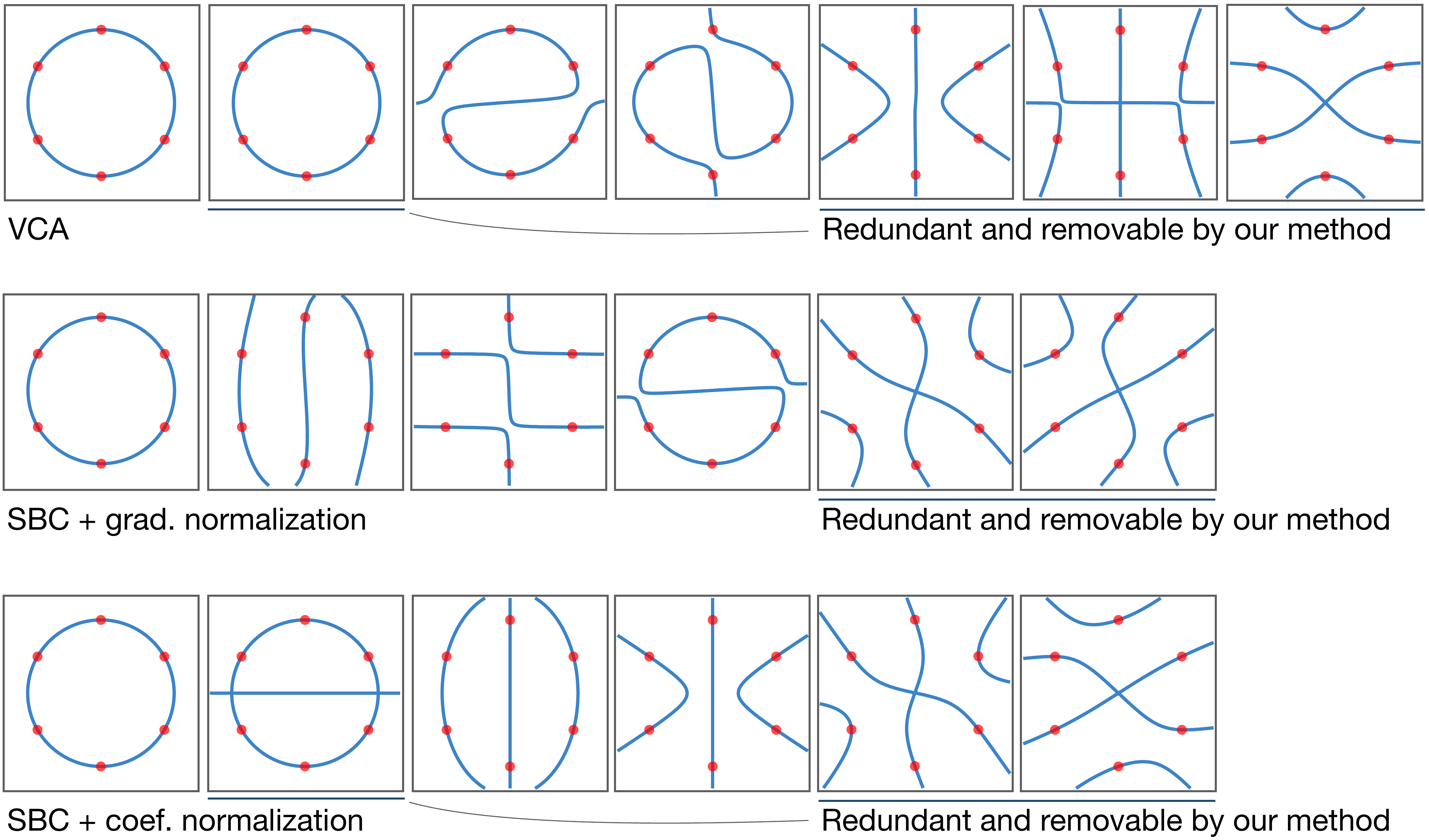}
\caption{Sets of vanishing polynomials obtained by VCA (top row), by SBC-$\mathfrak{n}_{\mathrm{g}}$ (middle row), and by SBC-$\mathfrak{n}_{\mathrm{c}}$ (bottom row) for six-point dataset in a noise-free case. All sets contain redundant basis polynomials, which can be efficiently removed by the basis reduction.}\label{fig:redundancy-6points}
\end{figure*}

\begin{figure*}
\includegraphics[scale=0.28]{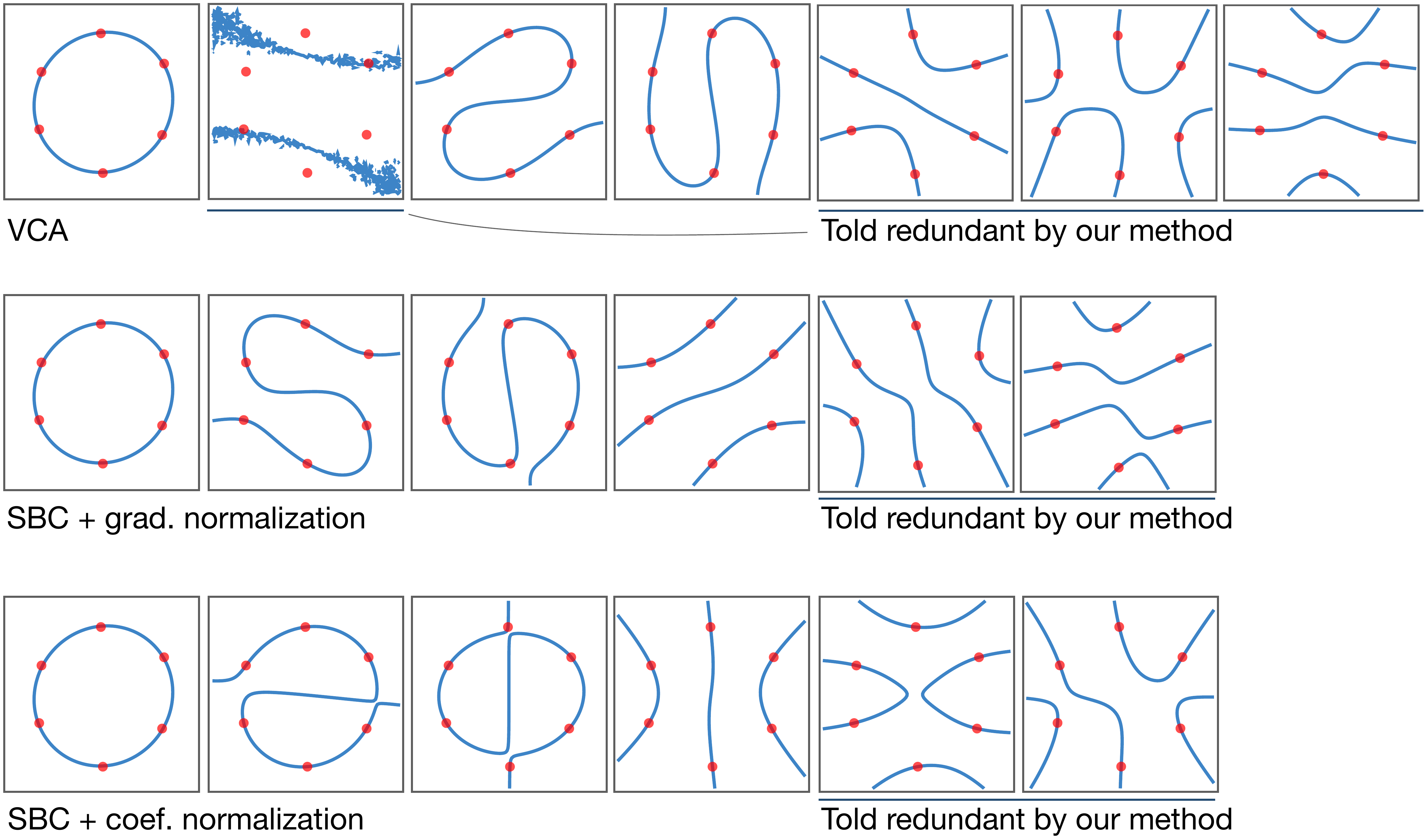}
\caption{Sets of vanishing polynomials obtained by VCA (top row), by SBC-$\mathfrak{n}_{\mathrm{g}}$ (middle row), and by SBC-$\mathfrak{n}_{\mathrm{c}}$ (bottom row) for six-point dataset in a noisy case. All sets contain polynomials that are suggested as redundant by the basis reduction.}\label{fig:redundancy-6points-noisy}
\end{figure*}

\bibliography{main}
\bibliographystyle{plain}